\documentclass{article}

\usepackage{fullpage}
\usepackage[dvips]{graphicx}
\usepackage[cmex10]{amsmath}
\usepackage{amssymb}
\usepackage{amsthm}
\usepackage{subfigure}
\usepackage{multirow}
\usepackage{float}
\usepackage{color}

\usepackage{atbegshi}
\AtBeginDocument{\AtBeginShipoutNext{\AtBeginShipoutDiscard}}

\DeclareMathOperator*{\argmin}{\arg\,\min}

\newcommand{\mcF}{\mathcal{F}}
\newcommand{\mcL}{\mathcal{L}}

\newcommand{\given}{\mid}
\newcommand{\E}{\mathbb{E}}
\newcommand{\Yh}{\hat{Y}}

\newcommand{\abs}[1]{\lvert#1\rvert}

\newtheorem{theorem}{Theorem}
\newtheorem{lemma}[theorem]{Lemma}
\newtheorem{cor}[theorem]{Corollary}

\date{}

\begin{document}

\title{Distribution-Preserving k-Anonymity}
\author{Dennis~Wei, Karthikeyan~Natesan~Ramamurthy, and~Kush~R.~Varshney\\
IBM Research AI\\
Thomas J.\ Watson Research Center, Yorktown Heights, NY}
\thanks{Portions of this work were first presented at the 2015 SIAM International Conference on Data Mining.}

\maketitle

\begin{abstract}
Preserving the privacy of individuals by protecting their sensitive attributes is an important consideration during microdata release. However, it is equally important to preserve the quality or utility of the data for at least some targeted workloads. We propose a novel framework for privacy preservation based on the k-anonymity model that 
is ideally suited for workloads that require preserving 
the probability distribution of the quasi-identifier variables in the data. Our framework 
combines the principles of distribution-preserving quantization and k-member clustering, and we specialize it to 
two variants that respectively use intra-cluster and Gaussian dithering of cluster centers to achieve distribution preservation. We perform theoretical analysis of the proposed schemes in terms of distribution preservation, and describe their utility in workloads such as covariate shift and transfer learning where such a property 
is necessary. Using extensive experiments on real-world Medical Expenditure Panel Survey data, we demonstrate the merits of our algorithms over standard k-anonymization for a hallmark health care application where an insurance company wishes to understand the risk in entering a new market. Furthermore, by empirically quantifying the reidentification risk, we also show that the proposed approaches indeed maintain k-anonymity.
\end{abstract}



\section{Introduction}
\label{sec:introduction}

Data owners often employ data analysts to provide them with accurate and actionable insights.  In many important domains, the data to be analyzed consists of individual records with sensitive fields that must be protected to maintain privacy.  The passing of data from owner to analyst, known as microdata release or publishing, necessitates anonymization or some other similar privacy protection operation in these applications.  A commonly used privacy criterion in microdata publishing is $k$-anonymity \cite{Sweeney2002,Samarati2001}.  Any operation on a data set whose result achieves $k$-anonymity is equally good from the privacy perspective; it is the \emph{workload} for which the data is to be used that defines the quality or utility of the operation \cite{Samarati2001,Iyengar2002,YiWJ2014}.  

A key tool for the data analyst in developing insights is supervised learning.  In certain challenging settings, labeled training data is available to the supervised learning algorithm for one problem or set of conditions, but testing is to be performed on a different problem or set of conditions.  One common way of addressing this challenge is by \emph{transfer learning} and \emph{covariate shift} methods that reweight training samples in accordance with the probability distributions of the different problems \cite{SugiyamaKM2007,BickelSS2009,Quionero-CandelaSSL2009}.  In this paper, our contribution is to develop the first (to the best of our knowledge) $k$-anonymity approach for workloads such as transfer learning and covariate shift that require the preservation of the probability distributions of the data \cite{WeiRV2015}.

Consider a number of colleges with different student distributions that want to release their data to enable the accurate learning of predictive models for a common learning outcome \cite{BalkanBFM2012}.  Using the data from all of the colleges together allows for better learning than separately due to the advantages of transfer learning, but releasing such information is fraught with privacy issues such as satisfying the Family Educational Rights and Privacy Act (FERPA) in the United States \cite{DariesRWYWHSC2014}.  There are similar laws in applications besides education, including the  Health Insurance Portability and Accountability Act (HIPAA) in health care and the Gramm-Leach-Bliley Act (GLBA) in personal finance.  Several statistical interpretations of the legal language for protecting privacy exist; the property of $k$-anonymity is a common interpretation to lower the risk of reidentification \cite{MalinBM2011,ElemamD2008}.  

As another example, consider the change in the landscape of health insurance in the United States after the passage of the Patient Protection and Affordable Care Act.  Health insurance companies entered new markets, defined by geography, by age group, and by other prospect base criteria.  When the legislation was being enacted, the companies had to decide which new markets to enter using the information at their disposal at the time.  In making the decisions, companies sought to enter markets containing an abundance of profitable, i.e.\ low-cost, individuals likely to enroll in their plans.  This objective gives rise to the problem of market risk assessment, the estimation of cost profiles of new market populations.  Successful market risk assessment, however, faces at least two challenges.  On the privacy front, even internal use of individual health data by an insurance company for its planning and strategy is protected by HIPAA.  On the supervised learning front, the companies had access to member health cost data from their existing markets available for training, but no such cost data available for the new markets, necessitating predictive models with covariate shift.

In privacy, there are three types of variables: key attributes, quasi-identifiers, and sensitive attributes.  Key attributes, such as the name of the person, are always dropped before microdata release.  Quasi-identifiers are often demographic variables like age, gender, and postal code that can be matched to other publicly-available data sets such as voter registration records containing both key attributes and quasi-identifiers to reveal identities in the data set to be protected.  It is this type of identity disclosure that $k$-anonymity aims to protect against.  Under the $k$-anonymity privacy model, the quasi-identifiers for an individual cannot be distinguished from the quasi-identifiers of at least $k-1$ other individuals \cite{Sweeney2002,Samarati2001}.  The sensitive attributes are variables such as educational or medical test results.  

There are many anonymization algorithms to transform the quasi-identifier attributes of a data set to achieve $k$-anonymity, but these existing algorithms do not seek to preserve the data distribution, as is required for covariate shift and transfer learning.  Existing algorithms for $k$-anonymity include generalizations and suppressions \cite{Iyengar2002}, multidimensional generalization \cite{LeFevreDR2006}, and multidimensional clustering in which the samples or records in the data are grouped by similarity such that the smallest group has at least $k$ elements \cite{ByunKBL2007}.  Clustering approaches offer the most flexibility and best performance.

The reason that existing clustering approaches to anonymization are not distribution-preserving is as follows: the optimization criterion is based on the average distortion of the individual samples and with such a criterion, the optimal cluster centers do not follow the same distribution as the original data.  In the asymptotic limit as the number of clusters goes to infinity, the distribution of the cluster center locations  is the original data distribution to the one third power, properly normalized \cite{GrayN1998}.

Distribution-preserving quantization is an alternative to standard clustering methods that does have 
the desired output behavior \cite{LiKK2010,AlamgirLL2014}.  It has been developed in the context of audio signal processing and has never been considered in the privacy preservation context before.  Specifically, the approach of \cite{LiKK2010} is based on subtractive dithered quantization \cite{LipshitzWV1992} followed by Rosenblatt's transformation \cite{Rosenblatt1952}.  We emphasize that although dithering, the introduction of noise or random perturbations, is fraught with several issues when used for privacy preservation \cite{KarguptaDWS2003}, in our work, the introduction of noise is not for anonymization purposes, but to allow the manipulation of the distribution.  Since distribution-preserving quantization comes from the signal processing and communications domain, not the privacy domain, it does not aim to achieve $k$-anonymity.  In particular, like $k$-means clustering and other standard clustering methods, it takes the number of clusters as an input parameter rather than the minimum number of samples in each cluster, which is what is needed for $k$-anonymity.  

Clustering for $k$-anonymity requires a different problem setup, given the name $k$-member clustering in \cite{ByunKBL2007}.  The $k$-member clustering problem has the $k$ of $k$-anonymity as the input parameter rather than the number of clusters.  Algorithms for $k$-member clustering may be divided into two classes: objective-driven optimization algorithms and simple scalable algorithms.  The first class includes algorithms based on minimum-cost network flow \cite{DemirizBB2009} and a cluster penalty function \cite{Rebollo-MonederoFPP2013}.  The second class includes greedy clustering \cite{ByunKBL2007}, subsampling and local optimization \cite{BanerjeeG2006}, and constrained agglomerative clustering \cite{GeEJD2007}.  This problem is also related to clustering with maximum cluster size constraints \cite{GeethaPV2009,GanganathCT2014}, semi-supervised clustering \cite{BasuDW2009}, and maximum output entropy quantization \cite{Messerschmitt1971}.

\emph{Contribution.}  To the best of our knowledge, there is no existing approach that transforms quasi-identifier data to achieve $k$-anonymity while also preserving its probability distribution.  Our solution combines the key aspects of $k$-member clustering algorithms with the key aspects of distribution-preserving quantization to obtain a distribution-preserving $k$-anonymity transformation.  We propose two ways to dither cluster centers: intra-cluster dither and Gaussian dither.  For each, we theoretically analyze the distribution-preserving properties.  As the distribution preservation is motivated by workloads such as covariate shift and transfer learning that require the data distribution, we detail how our proposed method can be used in conjunction with those machine learning tasks.  We demonstrate the proposed approach on real-world Medical Expenditure Panel Survey (MEPS) data for the health insurance market risk assessment application mentioned above.  We find that the proposed method does in fact maintain $k$-anonymity empirically by investigating reidentification risk and find that the distribution preservation is critical to obtaining usable machine learning predictions.

The remainder of the paper is organized as follows.  In Section~\ref{sec:background}, we provide background on the transfer learning and covariate shift problems encountered in supervised learning.  In Section~\ref{sec:privacy}, we develop the new privacy-preservation method for the workloads of interest that combines aspects of $k$-member clustering and distribution-preserving quantization.  In Section~\ref{sec:empirical}, we provide empirical results on real-world health care data for market risk assessment.  Section~\ref{sec:conclusion} provides a summary and discussion.

\section{Covariate Shift and Transfer Learning}
\label{sec:background}

In this section, we first introduce notation, then describe the basic learning problem encountered in the covariate shift setting, and finally describe the problem in the transfer learning setting.  Random variables are indicated by capital letters and their samples by lowercase letters.  $\E$ is used to denote the expectation of a random variable.

Consider the following general supervised learning problem: we wish to predict a response variable $Y$ using predictor variables $X$.  Given a class of functions $\mcF$ and training samples $(x_i, y_i)$, $i=1,\dots,n$, a predictor function is selected from $\mcF$ to minimize the empirical risk, 
\begin{equation}\label{eqn:riskEmp}
\Yh(\cdot) = \argmin_{f\in\mcF} \; \frac{1}{n} \sum_{i=1}^{n} \mcL(f(x_i), y_i), 
\end{equation}
for some choice of loss function $\mcL$ that measures the error between the predicted response $f(x_i)$ and actual response $y_i$.  

Assume that the training samples are drawn i.i.d.\ from the joint distribution $p_{X,Y} = p_X p_{Y\given X}$. The problem of \emph{covariate shift} occurs when the predictor variables or covariates are drawn from a different distribution $q_X$ in the test phase.  The conditional distribution $p_{Y\given X}$ is assumed to remain the same.  In the most idealized setting, covariate shift does not necessarily pose a problem:  As the number of samples $n \to \infty$, the empirical risk in \eqref{eqn:riskEmp} converges to the population risk 
\[
\E \left[ \mcL(f(X), Y) \right] = \E \left[ \E \left[ \mcL(f(X), Y) \mid X \right] \right],
\]
from which it can be seen that the optimal choice of predictor $f$ depends only on the conditional distribution $p_{Y\given X}$, regardless of the marginal distribution for $X$ (e.g.~$p_X$ or $q_X$).  Hence as $n \to \infty$, the conditional distribution $p_{Y\given X}$ can be learned very accurately and the optimal predictor can be obtained provided that the class $\mcF$ is rich enough to contain it.  However, in practical settings where $n$ is finite or $\mcF$ is overly constrained, then the predictor $\Yh$ resulting from \eqref{eqn:riskEmp} generally depends on the training distribution $p_X$ and thus can be mismatched to the test distribution $q_X$ under which performance is evaluated. 

A straightforward solution to covariate shift is to weight the training samples by the ratio 
\[
w(x) = \frac{q_X(x)}{p_X(x)}.
\]
This weighting represents the relative importance of each sample under $q_X$ rather than $p_X$. The weighted empirical risk 
\[
\frac{1}{n} \sum_{i=1}^{n} w(x_i) \mcL(f(x_i), y_i)
\]
then converges to
\begin{equation}\label{eqn:riskPopW}
\E_{p_X p_{Y\given X}} \left[ w(X) \mcL(f(X), Y) \right]
= \E_{q_X p_{Y\given X}} \left[ \mcL(f(X), Y) \right],
\end{equation}
thus matching the test distribution. 

In the transfer learning problem, we have $m$ different tasks each with their own joint distributions in training and testing: $p_{X,Y\given T}$ and $q_{X,Y\given T}$, where $T \in \{1,\ldots,m\}$ is a variable indicating the task.  The tasks may be related to each other and we would like to use the training samples from all other tasks in helping learn $f_t$, the predictor for task $t$.  Like covariate shift, the transfer learning problem may be addressed by distribution matching through the weight
\begin{displaymath}
w(x,y \given T = t) = \frac{p_{X,Y\given T}(x,y\given T = t)}{\sum_{t'=1}^m p_T(T=t')p_{X,Y\given T}(x,y\given T = t')} \frac{q_{X\given T}(x\given T = t)}{p_{X\given T}(x \given T = t)}.
\end{displaymath}
Note that the second part of the weight is the same as the covariate shift weight to take differences between training and testing into account.  The first part of the weight allows for the transfer of information across tasks.  With this weight, we have convergence to \cite{BickelSS2009}:
\begin{equation}\label{eqn:riskPopWtransfer}
\E_{\sum_{t'=1}^{m}p_{t'}p_{X\given t'} p_{Y\given X,t'}} \left[ w(X,Y\given t) \mcL(f_t(X), Y) \right] = \E_{q_{X\given t} p_{Y\given X,t}} \left[ \mcL(f_t(X), Y) \right],
\end{equation}
thereby matching the test distribution of task $t$.

We must also estimate the weights $w(x)$ and $w(x,y\given t)$ that we have defined in the learning process.  There are a variety of ways to estimate them including nonparametric methods and methods based on logistic regression, cf.\ \cite{WeiRV2015,BickelSS2009}.  We do not go into the details here, but only emphasize that the data distributions are the key ingredient when estimating the weights.  Therefore, any privacy preservation operation performed on the data should ideally preserve the data distributions.

\section{Distribution-Preserving $k$-Anonymity}
\label{sec:privacy}


As discussed in the introduction, the privacy of individuals must be protected when working with their personal 
data in domains such as education and health care.  In particular, taking $k$-anonymity as the notion of privacy, the quasi-identifiers $x$ in 
the original data 
must be converted to some other values $\hat{x}$ 
in a way that the data for an individual cannot be distinguished from at least $k-1$ others.  (For notational simplicity, we use $x, y$ generically in this section to refer to training data distributed as $p_{X,Y}$ in the standard supervised learning or covariate shift settings, or to each training data set distributed as $p_{X,Y\given T=t}$, $t = 1,\dots,m$ in the transfer learning setting.)  
Moreover, given 
our ultimate goal of predicting response $Y$, 
we not only want the samples $\hat{x}_i$, 
$i=1,\ldots,n$ to have the $k$-anonymity property, but also 
the 
model learned from $(\hat{x}_i,y_i)$, $i=1,\ldots,n$ to have small prediction error, as quantified by relative bias, 
$R^2$, or 
other measures of generalization.  

With these dual goals in mind, in \cite{WeiRV2015} we introduced 
a sequence of operations inspired by $k$-member clustering \cite{ByunKBL2007} and distribution-preserving quantization with dithering and transformation \cite{LiKK2010}.
In the present paper, we generalize and further develop this framework.  Figure~\ref{fig:blockdiagram} illustrates the overall procedure.
\begin{figure}
\centering
\includegraphics[width=\columnwidth]{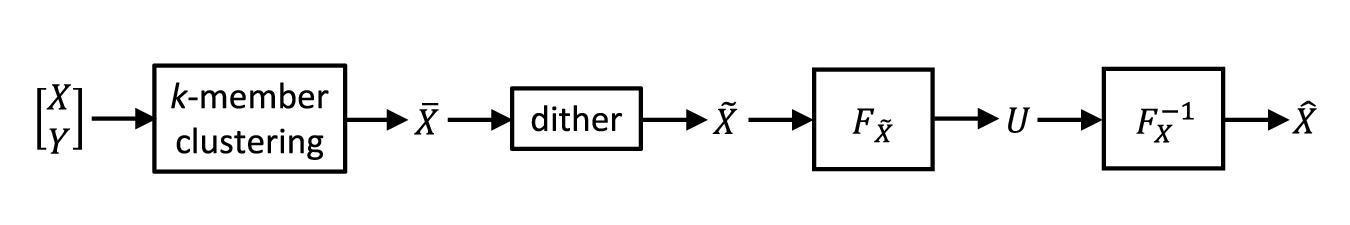} 
\caption{Block diagram of the operations to achieve $k$-anonymity and distribution preservation.}
\label{fig:blockdiagram}
\end{figure}
Section~\ref{sec:privacy:clustering} discusses the first step of $k$-member clustering while Section~\ref{sec:privacy:trans} discusses the subsequent distribution-recovering transformation.

\subsection{$k$-Member Clustering}
\label{sec:privacy:clustering}

The original data $(x_i, y_i)$, $i=1,\dots,n$ is first clustered subject to the $k$-anonymity requirement that each cluster contain at least $k$ members.  Define $c$ to be the number of clusters and let $\ell = 1,\dots,c$ index the clusters.  Let $a_{i\ell} = 1$ if sample $i$ is assigned to cluster $\ell$ and $a_{i\ell} = 0$ otherwise.  Denote by $(\bar{x}_\ell, \bar{y}_\ell)$ the centroid chosen to represent the samples in cluster $\ell$.  The centroids are determined in part by a function $d\left( (x_i, y_i), (\bar{x}_\ell, \bar{y}_\ell)\right)$ that measures the distortion between two points.  
Here we choose $d$ 
to be a weighted squared Euclidean distance,
\begin{equation}\label{eq:distortionEuclidean}
d\left( (x_i, y_i), \bigl(\bar{x}_\ell, \bar{y}_\ell\bigr)\right) = \bigl\lVert x_i - \bar{x}_\ell \bigr\rVert_2^2 + w \bigl(y_i - \bar{y}_\ell\bigr)^2,
\end{equation}
where $w > 0$ is the weight on the response component relative to the quasi-identifier components, intended to account for any scale difference between them.  

With the foregoing definitions, $k$-member clustering can be formulated as the following optimization problem:
\begin{equation}\label{eq:kMemberClustering}
\begin{split}
\min_{\{a_{i\ell}\}, \{(\bar{x}_\ell, \bar{y}_\ell)\}} &\sum_{i=1}^{n} \sum_{\ell=1}^{c} a_{i\ell} d\left( (x_i, y_i), \bigl(\bar{x}_\ell, \bar{y}_\ell\bigr)\right)\\
\text{s.t.} \quad &a_{i\ell} \in \{0,1\} \quad \forall \; i,\ell,\\
&\sum_{\ell=1}^{c} a_{i\ell} = 1 \quad \forall \; i,\\
&\sum_{i=1}^{n} a_{i\ell} \geq k \quad \forall \; \ell.
\end{split}
\end{equation}
The last line in \eqref{eq:kMemberClustering} expresses the $k$-member constraint while the second-last line ensures that each sample is assigned to exactly one cluster.  These constraints imply that the number of clusters $c$ is at most $\lfloor\frac{n}{k}\rfloor$.  In this work, we simply fix $c = \lfloor\frac{n}{k}\rfloor$; optimization over the number of clusters is a subject for future study. 

Solutions to the $k$-member clustering problem \eqref{eq:kMemberClustering} can be obtained using methods such as \cite{DemirizBB2009,Rebollo-MonederoFPP2013,ByunKBL2007,BanerjeeG2006,GeEJD2007} discussed in Section~\ref{sec:introduction}.  In the experiments of Section~\ref{sec:empirical}, due to its scalability, we use the greedy algorithm proposed in \cite{ByunKBL2007} modified to use the Euclidean distortion function \eqref{eq:distortionEuclidean}.  

Once a solution to \eqref{eq:kMemberClustering} has been determined, the original response values $y_i$ are set aside, to be rejoined only at the end with the transformed quasi-identifiers $\hat{x}_i$. 
In place of each original $x_i$, we record only the cluster $\ell_i$ to which it was assigned, i.e.\ $\ell_i = \ell$ for which $a_{i\ell} = 1$.  For each cluster $\ell$, the associated set of values $\mathcal{V}_\ell = \{x_i, \; i : a_{i\ell}=1 \}$ is retained; 
the centroid $\bar{x}_\ell$ is a function of $\mathcal{V}_\ell$ 
and may or may not be retained depending on the type of dither to be applied.  Importantly however, an index $i$ is no longer linked to any single value in $\mathcal{V}_{\ell_i}$.  
In this way the clustering step achieves our goal of $k$-anonymity; since all 
operations that follow depend only on the whole sets $\mathcal{V}_{\ell_i}$, 
a level of privacy equivalent to $k$-anonymity is maintained. 
In Section~\ref{sec:empirical:reID}, we confirm through reidentification experiments that the proposed distribution-recovering operations not only preserve $k$-anonymity but can also result in even stronger privacy.

\subsection{Distribution-Recovering Transformation}
\label{sec:privacy:trans}

The operations following $k$-member clustering in Figure~\ref{fig:blockdiagram} aim to reconstruct a $k$-anonymous data set that follows the distribution of $X$.  Using the clustering output described at the end of Section~\ref{sec:privacy:clustering}, the first 
operation of dithering (the intentional application of noise) produces $\tilde{x}_i$, $i=1,\dots,n$, sampled from a continuous probability distribution.  (By contrast, simply using the centroids $\bar{x}_\ell$ would result in a discrete distribution.) 
Different types of dither give rise to different variants of the proposed procedure.  This makes our framework quite general since any dither with a continuous distribution may be used.  In Sections~\ref{sec:privacy:intra} and \ref{sec:privacy:gaussian}, we discuss two versions corresponding to intra-cluster and Gaussian dither.

The final two operations in Figure~\ref{fig:blockdiagram} constitute Rosenblatt's transformation \cite{Rosenblatt1952} for transforming an arbitrary continuous distribution into a specified target distribution, in our case that of $X$.  The first of these operations, transformation to a uniform distribution, is given by the following sequence of (conditional) cumulative distribution functions (CDF) of $\tilde{X}$, one for each dimension:
\begin{equation}\label{eq:RosenblattForward}
\begin{split}
	u_{i1} &= F_{\tilde{X}_1}(\tilde{x}_{i1}) = F_{\tilde{X}_1\given\tilde{X}^0}\bigl(\tilde{x}_{i1} \given \tilde{x}_i^0 \bigr)\\
	u_{i2} &= F_{\tilde{X}_2\given\tilde{X}_1}\bigl(\tilde{x}_{i2}\given \tilde{x}_{i1}\bigr) = F_{\tilde{X}_2\given\tilde{X}^1}\bigl(\tilde{x}_{i2}\given \tilde{x}_i^{1}\bigr)\\
	&\vdots\\
u_{id} &= F_{\tilde{X}_d\given\tilde{X}^{d-1}}\bigl(\tilde{x}_{id} \given \tilde{x}_i^{d-1}\bigr),
\end{split}
\end{equation}
where $\tilde{x}_i^{j-1}$ is shorthand for $\tilde{x}_{i1}, \dots, \tilde{x}_{i,j-1}$ (similarly for other multi-dimensional quantities) and a superscript of $0$ is understood to indicate the empty list.  The detailed forms of the CDFs in \eqref{eq:RosenblattForward} depend on the choice of dither and are specified in Sections~\ref{sec:privacy:intra} and \ref{sec:privacy:gaussian}.  Transformation \eqref{eq:RosenblattForward} is applied to all samples $i = 1,\dots,n$ and can be done in parallel, for example cluster by cluster.

The last operation, transformation from a uniform distribution to the distribution of $X$, is similar to \eqref{eq:RosenblattForward} but with the inverse CDF of $X$:
\begin{equation}\label{eq:RosenblattInverse}
\begin{split}
	\hat{x}_{i1} &= F_{X_1}^{-1}(u_{i1}) = F_{X_1\given X^0}^{-1}\bigl(u_{i1} \given \hat{x}_i^0\bigr) \\
	\hat{x}_{i2} &= F_{X_2\given X_1}^{-1}\bigl(u_{i2}\given \hat{x}_{i1}\bigr) = F_{X_2\given X^1}^{-1}\bigl(u_{i2}\given \hat{x}_i^1\bigr)\\
	&\vdots\\
	\hat{x}_{id} &= F_{X_d\given X^{d-1}}^{-1}(u_{id}\given \hat{x}_i^{d-1}),\end{split}
\end{equation}
where the generalized inverse CDF is defined as 
\begin{equation}\label{eq:invCDF}
F_X^{-1}(u) = \inf\{x : F_X(x) \geq u\}.
\end{equation}
In practice, the underlying distribution that generates the data $x$ is not known.  Instead we use the empirical distribution, in which case $X$ can always be regarded as discrete.  For each dimension $j = 1, \dots, d$, define $v_j(1) < v_j(2) < \dots < v_j(n_j)$, $n_j \leq n$, to be the distinct observed values of $X_j$ in increasing order.  Then by specializing \eqref{eq:invCDF} to discrete distributions, transformation \eqref{eq:RosenblattInverse} can be expressed more explicitly as 
\begin{align}
&\hat{x}_{ij} = v_j(i_j) \quad \text{if}\nonumber\\
& F_{X_j\given X^{j-1}}\bigl(v_j(i_j-1) \given \hat{x}_i^{j-1}\bigr) 
< u_{ij} 
\leq F_{X_j\given X^{j-1}}\bigl(v_j(i_j) \given \hat{x}_i^{j-1}\bigr), \nonumber\\ 
&\qquad\qquad\qquad\qquad i_j = 1,\dots,n_j, \quad j = 1,\dots,d,\label{eq:RosenblattInverse2}
\end{align}
where we define $v_j(0) = -\infty$ for $i_j = 1$ so that $F_{X_j\given X^{j-1}}(v_j(0) \given \hat{x}_i^{j-1}) = 0$.

At the end of the process, the sensitive $y_i$ values are rejoined with the clustered and transformed quasi-identifiers $\hat{x}_i$. 
Overall, this sequence of steps yields output samples $(\hat{x}_i,y_i)$, $i=1,\dots,n$ that are close 
to the original samples $(x_i,y_i)$ in distribution while being $k$-anonymous.  The main free parameter, $k$, can be varied to achieve the desired tradeoff between privacy and 
prediction error.

\subsubsection{Intra-cluster dither}
\label{sec:privacy:intra}

In this subsection, we consider dither distributions that are supported only on the cluster to which a sample belongs.  In other words, for each $i$, the cluster assignment $\ell_i$ is viewed as a conditioning event and $\tilde{x}_i$ is sampled randomly from the support of $\ell_i$ with probability one.  Here the support 
can be any subset of $\mathbb{R}^d$ that contains the points in $\mathcal{V}_{\ell_i}$ and no others.    

The framework of Figure~\ref{fig:blockdiagram} with intra-cluster dither as defined above is still quite general since it leaves open the exact specification of support sets and corresponding probability distributions.  In the following we show that it encompasses the important special case of random resampling with replacement from the values within each cluster.  In this 
approach, each transformed sample $\hat{x}_i$ is drawn randomly from the value set $\mathcal{V}_{\ell_i}$ of the corresponding cluster.  A closely related alternative is resampling \emph{without} replacement, in which 
$\hat{x}_i$ corresponding to the same cluster $\ell_i$ are chosen as a random permutation of $\mathcal{V}_{\ell_i}$.  

We focus here on resampling with replacement under the constraint that $\hat{X}$ preserves the distribution of $X$, as in Figure~\ref{fig:blockdiagram}.  This constraint can be met by selecting values within a cluster according to their empirical frequencies, as stated below in Lemma~\ref{lem:resampling}. 
First define $v(i_1,\dots,i_d) = \bigl(v_1(i_1),\dots,v_d(i_d)\bigr)$ to be the multi-dimensional extension of $v_j(i_j)$, noting that some of these combinations of values may not be observed. 
Denote by $n_\ell(i_1,\dots,i_d)$ the number of samples in cluster $\ell$ with value $v(i_1,\dots,i_d)$; $n(i_1,\dots,i_d)$ (without the subscript $\ell$) denotes the corresponding number in all clusters; $n_\ell$ (without the indices $i_j$) denotes the number of samples in cluster $\ell$ ($n_\ell \geq k$ from \eqref{eq:kMemberClustering}); and $n$ is the total number of samples as before. 
Some of these definitions are illustrated in Figure~\ref{fig:partRect} below.  
\begin{lemma}\label{lem:resampling}
If for all samples $i$ and conditioned on cluster $\ell_i$, $\hat{x}_i = v(i_1,\dots,i_d)$ is chosen with probability $n_{\ell_i}(i_1,\dots,i_d) / n_{\ell_i}$, then the transformed quasi-identifiers $\hat{X}$ have the same distribution as $X$. 
\end{lemma}
\begin{proof}
The lemma follows from calculating 
the probability mass function (PMF, denoted generically by $p$) of $\hat{X}$.  Treating cluster membership as a conditioning random variable $L$, we have  
\begin{align*}
p_{\hat{X}}(v(i_1,\dots,i_d)) &= \sum_{\ell=1}^{c} p_L(\ell) p_{\hat{X}\given L}\bigl(v(i_1,\dots,i_d) \given \ell \bigr)\\
&= \sum_{\ell=1}^{c} \frac{n_\ell}{n} \frac{n_\ell(i_1,\dots,i_d)}{n_\ell}\\
&= \frac{n(i_1,\dots,i_d)}{n}\\
&= p_{X}(v(i_1,\dots,i_d)),
\end{align*}
using 
the definition of empirical distribution in the last line. 
\end{proof}

To establish the equivalence between the proposed procedure and resampling with replacement, 
we construct a rectangular partition of $\mathbb{R}^d$ and define support sets accordingly. 
Figure~\ref{fig:partRect} provides an illustration in the two-dimensional case $d=2$.  For $j = 1,\dots,d$, let $\{\mathcal{I}_j(i_j)\}_{i_j=1}^{n_j}$ be a set of intervals that partition the real line such that each $\mathcal{I}_j(i_j)$ contains the observed value $v_j(i_j)$.  These intervals could be defined for example by the midpoints between the $v_j(i_j)$.  The full $d$-dimensional space is partitioned into rectangular cells $\mathcal{C}(i_1,\dots,i_d) = \mathcal{I}_1(i_1) \times \mathcal{I}_2(i_2) \times \dots \times \mathcal{I}_d(i_d)$ formed by Cartesian products of intervals.  Each cell $\mathcal{C}(i_1,\dots,i_d)$ contains a single 
$v(i_1,\dots,i_d)$. 
The \emph{support} $\mathcal{C}_\ell$ of cluster $\ell$ is then defined as the union of cells $\mathcal{C}(i_1,\dots,i_d)$ such that the corresponding $v(i_1,\dots,i_d) \in \mathcal{V}_\ell$.  

\begin{figure}
\centering
\includegraphics[width=0.5\columnwidth]{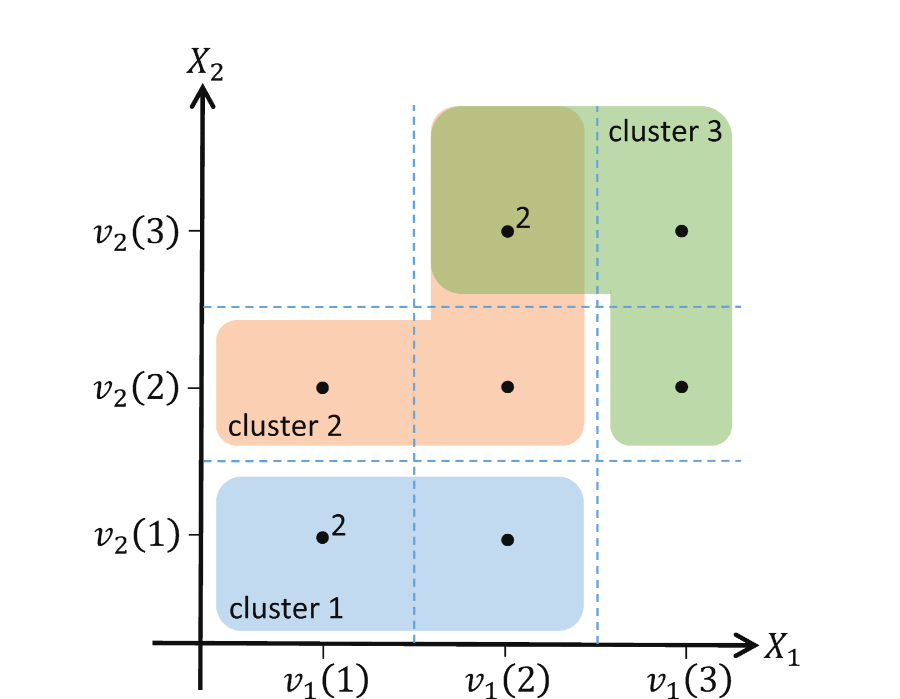} 
\caption{Partition of a two-dimensional ($d=2$) quasi-identifier space into rectangular cells $\mathcal{C}(i_1, i_2)$, each containing a single value $v(i_1, i_2)$.  Different shadings indicate the support sets of different clusters, which may overlap.  Cells $\mathcal{C}(1,1)$ and $\mathcal{C}(2,3)$ each contain $n(1,1) = n(2,3) = 2$ samples; all other cells with dots contain $n(i_1,i_2)=1$ sample. In $(1,1)$, both samples belong to cluster $1$ ($n_1(1,1) = 2$), while in $(2,3)$, the samples are split between clusters $2$ and $3$ ($n_2(2,3) = n_3(2,3) = 1$).}
\label{fig:partRect}
\end{figure}

We now define a dither distribution by the 
probabilities 
\begin{equation}\label{eq:probCell}
\Pr\bigl(\tilde{X} \in \mathcal{C}(i_1,\dots,i_d) \given \ell \bigr) = \frac{n_{\ell}(i_1,\dots,i_d)}{n_\ell}. 
\end{equation}
Conditioned on $\tilde{X} \in \mathcal{C}(i_1,\dots,i_d)$ (and independent of cluster $\ell$), let $\tilde{X}$ be uniformly distributed over $\mathcal{C}(i_1,\dots,i_d)$.

\begin{theorem}\label{thm:resampling}
For dither $\tilde{X}$ with piecewise-uniform conditional distribution defined by \eqref{eq:probCell}, the procedure in Figure~\ref{fig:blockdiagram} is equivalent to resampling with replacement from the cluster values in $\mathcal{V}_\ell$ with probabilities $n_{\ell}(i_1,\dots,i_d) / n_\ell$.
\end{theorem}
\begin{proof}
Please see Appendix~\ref{app:thm}.
\end{proof}

In the proof of Theorem~\ref{thm:resampling}, specifically in going from \eqref{eq:thm:resampling3} to \eqref{eq:thm:resampling4}, the assumption of a piecewise-uniform dither distribution is not strictly necessary.  
We have chosen to make this assumption because it simplifies the proof and interpretation.  Furthermore, in the one-dimensional case $d=1$, the proof makes no reference to uniformity or any property of the cell-conditional distributions other than continuity.  In fact, the cell probability requirements \eqref{eq:probCell} can be further relaxed, specifically by merging contiguous cells that all belong in full to the same cluster, i.e., cells such that $n(i_1) = n_\ell(i_1)$ for a single $\ell$.  We index these merged cells by $i'_1$ and denote by $n_\ell(i'_1)$ the number of samples in cluster $\ell$ and cell $i'_1$.  Let the dither distribution satisfy 
\begin{equation}\label{eq:probCell1}
\Pr\bigl(\tilde{X} \in \mathcal{C}(i'_1) \given \ell \bigr) = \frac{n_{\ell}(i'_1)}{n_\ell} \quad \forall \; i'_1. 
\end{equation}
\begin{cor}\label{cor:resampling}
In the one-dimensional case $d=1$, Theorem~\ref{thm:resampling} holds if the dither $\tilde{X}$ satisfies \eqref{eq:probCell1} for cells $i'_1$ formed by first isolating unique observed values $v_1(i_1)$ and then merging contiguous cells that belong in full to the same cluster, as described above.
\end{cor}
\begin{proof}
Following the proof of Theorem~\ref{thm:resampling}, if $\tilde{x}$ belongs to a cell $\mathcal{C}(i'_1)$, then $u = F_{\tilde{X}}(\tilde{x})$ lies in an interval that maps to the set of observed values $v_1(i_1)$ contained in $\mathcal{C}(i'_1)$, analogous to \eqref{eq:thm:resampling2}.  Since Rosenblatt's transformation ensures that $u$ is uniformly distributed, the individual values $v_1(i_1)$ are then selected according to their empirical frequency.
\end{proof}


\subsubsection{Gaussian dither}
\label{sec:privacy:gaussian}

In this subsection, we assume that the centroid $\bar{x}_\ell$ is the mean of the samples in cluster $\ell$, which results when the weighted Euclidean distortion function \eqref{eq:distortionEuclidean} is chosen.  For each cluster $\ell$, we also compute 
the empirical covariance matrix $\Sigma_\ell$ 
from the $n_\ell$ samples in the cluster. 
Conditioned on cluster $\ell_i$, we then generate $\tilde{x}_i$ as an i.i.d.\ sample from the Gaussian distribution $\mathcal{N}\left(\bar{x}_{\ell_i}, \Sigma_{\ell_i} 
+ \alpha I \right)$. 
The extra bit of covariance with parameter $\alpha > 0$ accounts for rank-deficient $\Sigma_\ell$, 
which occurs when the number of distinct values in the cluster (bounded by $n_\ell \geq k$) is less than the dimension $d$. 

%

The above specification implies that the dither $\tilde{X}$ is distributed as a Gaussian mixture with $c$ mixture components, after marginalizing over the clusters.  Closed-form expressions can therefore be given for the first half of Rosenblatt's transformation \eqref{eq:RosenblattForward}. 
Denote by $\Phi(z; \mu, \sigma^2)$ the Gaussian CDF parameterized by mean $\mu$ and variance $\sigma^2$, and let $\Lambda_\ell = \Sigma_\ell + \alpha I$.  Then for $j = 1$ we have 
\begin{equation}
\label{eq:gaussmixture1}
	F_{\tilde{X}_1}(\tilde{x}_1) = \sum_{\ell=1}^c p_L(\ell) F_{\tilde{X}_1\given L}(\tilde{x}_1\given \ell) = \sum_{\ell=1}^c \frac{n_\ell}{n} \Phi\bigl(\tilde{x}_1; \bar{x}_{\ell 1}, \Lambda_{\ell,1,1}\bigr), 
\end{equation}
where $\Lambda_{\ell,j,j}$ is the $(j,j)$ element of $\Lambda_\ell$.  For $j > 1$,
\begin{equation}\label{eq:gaussmixturej}
F_{\tilde{X}_j\given\tilde{X}^{j-1}}\bigl(\tilde{x}_j \given \tilde{x}^{j-1}\bigr) = \sum_{\ell=1}^c p_{L\given\tilde{X}^{j-1}}(\ell \given \tilde{x}^{j-1}) F_{\tilde{X}_j\given\tilde{X}^{j-1},L}\bigl(\tilde{x}_j \given \tilde{x}^{j-1}, \ell\bigr).
\end{equation}
Since $\tilde{X} \given L$ is multivariate Gaussian, $\tilde{X}_j \given \tilde{X}^{j-1},L$ is also Gaussian with parameters given by 
\begin{align*}
\mu_{\ell j} &= \bar{x}_{\ell j} - \Lambda_{\ell,j,1:j-1} \bigl(\Lambda_{\ell,1:j-1,1:j-1} \bigr)^{-1} \bigl(\tilde{x}^{j-1} - \bar{x}_\ell^{j-1}\bigr)\\
\sigma^2_{\ell j} &= \Lambda_{\ell,j,j} - \Lambda_{\ell,j,1:j-1} \bigl(\Lambda_{\ell,1:j-1,1:j-1} \bigr)^{-1} \Lambda_{\ell,1:j-1,j},
\end{align*}
where $\Lambda_{\ell,j,1:j-1}$ is a row vector consisting of the first $j-1$ columns of the $j$th row of $\Lambda_\ell$ (analogously for $\Lambda_{\ell,1:j-1,j}$), and $\Lambda_{\ell,1:j-1,1:j-1}$ is the submatrix consisting of the first $j-1$ rows and columns of $\Lambda_{\ell}$.  Hence
\begin{equation}\label{eq:CDFgauss}
F_{\tilde{X}_j\given\tilde{X}^{j-1},L}\bigl(\tilde{x}_j \given \tilde{x}^{j-1}, \ell\bigr) = \Phi\bigl(\tilde{x}_j; \mu_{\ell j}, \sigma^2_{\ell j}\bigr).
\end{equation}
The conditional probabilities $p_{L\given\tilde{X}^{j-1}}(\ell \given \tilde{x}^{j-1})$ in \eqref{eq:gaussmixturej} can be computed recursively over $j=1,\dots,d$ using Bayes' rule:
\begin{align*}
p_{L\given\tilde{X}^{j}}(\ell \given \tilde{x}^{j})
&= \frac{p_{L\given\tilde{X}^{j-1}}(\ell \given \tilde{x}^{j-1}) f_{\tilde{X}_j\given\tilde{X}^{j-1},L}\bigl(\tilde{x}_j \given \tilde{x}^{j-1}, \ell\bigr)}{\sum_{\ell'=1}^{c} p_{L\given\tilde{X}^{j-1}}(\ell' \given \tilde{x}^{j-1}) f_{\tilde{X}_j\given\tilde{X}^{j-1},L}\bigl(\tilde{x}_j \given \tilde{x}^{j-1}, \ell'\bigr)}\\
&= \frac{p_{L\given\tilde{X}^{j-1}}(\ell \given \tilde{x}^{j-1}) \phi\bigl(\tilde{x}_j; \mu_{\ell j}, \sigma^2_{\ell j}\bigr)}{\sum_{\ell'=1}^{c} p_{L\given\tilde{X}^{j-1}}(\ell' \given \tilde{x}^{j-1}) \phi\bigl(\tilde{x}_j; \mu_{\ell' j}, \sigma^2_{\ell' j}\bigr)},
\end{align*}
where the second equality follows from \eqref{eq:CDFgauss} and $\phi$ denotes a Gaussian PDF.

\section{Empirical Results}
\label{sec:empirical}

In this section, we discuss empirical results obtained using 
real-world data motivated by the problem of health insurance market risk assessment.  Section~\ref{sec:mra} further describes 
the application.  Section~\ref{sec:empirical:data} discusses data sources and simulation of the different populations.  Section~\ref{sec:empirical:reID} assesses the reidentification risks of the proposed privacy preservation methods as applied to this data.  Section~\ref{sec:empirical:privacy_market} presents results on health care expenditure prediction under three scenarios:  with privacy constraints but without market shift, without privacy constraints but with market shift, and with both privacy constraints and market shift.

\subsection{Market Risk Assessment}
\label{sec:mra}

As mentioned in the introduction, market risk assessment involves estimating the health care cost of individuals in a new market who are likely to enroll in an insurer's plan.  These estimates can then be appropriately summarized into risk statistics for decision making.  Cost information for a new market is typically not available for competitive and other reasons.  Instead, insurers have access to cost data for their members in an existing market, plus demographic data for the existing and new markets from public sources, including age, gender, income, veteran status, smoking status, and place of residence.  This situation suggests a covariate shift approach as in Section~\ref{sec:background} in which a predictive model relating demographic variables, $X$, to health care cost, $Y$, is learned from the existing market population and is then applied to the new market's demographic data, taking into account the different demographic distributions in the existing and new markets, $p_X$ and $q_X$.  Since the cost $Y$ is regarded as being continuous, the learning problem \eqref{eqn:riskEmp} is one of regression.  Any regression technique can be used for this purpose, in particular those that account for the skewness and heteroscedasticity of health care costs, for example ordinary least-squares with log-transformed data, two-part models, generalized linear models, and multiplicative regression \cite{DiehrYAHL1999,BasuM2009,WeiRKM2014}. These correspond to different choices for the function class $\mcF$ and loss function $\mcL$ in \eqref{eqn:riskEmp}. 

However, market risk assessment differs from the standard covariate shift setting in Section~\ref{sec:background} in having a distinction between member populations and larger market populations, since only a subset of the latter enroll in a given health insurance plan.  To denote the different populations, we use the binary variables $E$ and $M$.  The variable $E$ indicates enrollment in an insurance company's plan ($E=1$ means enrolled), and the variable $M$ differentiates the existing market from the new market ($M=1$ means new market).  Since training data with costs comes from the insurance company's data on current plan members, the training distribution, previously denoted as $p_X$ in Section~\ref{sec:background}, is now denoted as $p_{X\given E,M}(x\given{e=1,m=0})$, referring to enrollees in the current market.  Likewise, the test distribution $q_X$ is $p_{X\given E,M}(x\given e=1,m=1)$ for enrollees in the new market.

In some cases it may be possible to estimate 
$p_{X\given E,M}(x \given 1,1)$ directly if enough is known about potential enrollees in the new market.  If this is true then the basic covariate shift framework in Section \ref{sec:background} applies. 
The more general and usual case is that $p_{X\given E,M}(x \given 1,1)$ cannot be estimated directly. 
In this scenario, besides the current member distribution $p_{X\given E,M}(x \given 1,0)$, we may still assume 
that demographic distributions for the current and new markets are available, corresponding to $p_{X\given M}(x\given 0)$ and $p_{X\given M}(x\given 1)$ respectively. These distributions are related by Bayes' rule,
\begin{equation}\label{eqn:BayesRule}
p_{X\given E,M}(x\given 1, m) = \frac{p_{E\given X,M}(1\given x,m) p_{X\given M}(x\given m)}{p_{E\given M}(1\given m)}, \quad m=0,1.
\end{equation}
Taking the ratio of $m=1$ to $m=0$ gives 
\begin{equation}\label{eqn:BayesRuleRatio}
\frac{p_{X\given E,M}(x\given 1, 1)}{p_{X\given E,M}(x\given 1, 0)} \propto \frac{p_{E\given X,M}(1\given x,1)}{p_{E\given X,M}(1\given x,0)} \frac{p_{X\given M}(x\given 1)}{p_{X\given M}(x\given 0)}
\end{equation}
as functions of $x$. 

In this work, we 
make the assumption that $p_{E\given X,M}(1\given x,m)$, the probability of enrollment conditioned on the predictor variables and market, is actually independent of the market $m$ once $x$ is fixed. In other words, $E$ and $M$ are conditionally independent given $X$ and $p_{E\given X,M}(1\given x,m) = p_{E\given X}(1\given x)$.  This is a reasonable starting 
assumption positing that enrollment depends 
on demographic variables such as age, sex, etc., but does not 
depend on which market the individual belongs to once those demographic variables are specified.  This assumption may be later modified by an insurance market expert to account for additional market factors such as the level of competition.  

With the conditional independence assumption, \eqref{eqn:BayesRuleRatio} simplifies to 
\[
p_{X\given E,M}(x\given 1, 1) \propto p_{X\given E,M}(x\given 1, 0) \frac{p_{X\given M}(x\given 1)}{p_{X\given M}(x\given 0)}.
\]
Since the training samples are distributed according to $p_{X\given E,M}(x\given 1, 0)$ while the test samples are distributed according to $p_{X\given E,M}(x\given 1, 1)$, the importance weighting is therefore $w(x) = p_{X\given M}(x\given 1) / p_{X\given M}(x\given 0)$, 
taking the place of $w(x) = q_X(x)/p_X(x)$ in Section \ref{sec:background}.  The importance weights may be estimated in the same manner as before, for example non-parametrically using empirical distributions, or using logistic regression.  We evaluate both of these covariate shift methods in Sections~\ref{sec:empirical:est} and \ref{sec:empirical:privacy}.   

A second notable feature of market risk assessment, and insurance applications in general, is the emphasis on 
%
%
aggregate predictions of the average or total cost for groups of individuals rather than individual-level predictions.  For example, one may be interested in 
the average cost for a new enrollee population as a whole or for segments of the population.  Given a regression model constructed as described above, aggregate predictions can be obtained simply by averaging the predictions for each individual in the group.  In terms of prediction error, this averaging has the effect of greatly attenuating the error variance: For a group of $m$ i.i.d.\ individuals, the error variance decreases by a factor of $m$.  As a consequence, for large $m$ the bias of the predictor, $b(\Yh) = \E[\Yh] - \E[Y]$, becomes the dominant measure of error as compared to variance or $R^2$, a common representation of individual-level MSE.  We refer the reader to \cite{WeiRV2015} for a straightforward calculation decomposing aggregate prediction error into bias and variance/$R^2$ components.  Thus in reporting prediction results in Section~\ref{sec:empirical:privacy_market}, we focus more on bias performance rather than variance or $R^2$. 

\subsection{Description of Data}
\label{sec:empirical:data}

We use publicly-available Medical Expenditure Panel Survey (MEPS)  data, which shares many characteristics with actual health cost data from insurance companies that we have worked with in the recent past but cannot include in this paper due to its confidentiality.  Based on large-scale surveys produced by the United States Department of Health and Human Services' Agency for Healthcare Research and Quality, MEPS contains the annual health care cost and demographic information of people across the United States.  However, since it does not come from an insurance company, there is neither a concept of a market in the data nor 
of enrollment in a company's plan.  
Thus in order to perform market risk assessment, we define two market populations and enrolled subsets of these populations as described below.  


We consider a scenario in which an insurance company is currently active in many areas that collectively are representative of the United States as a whole.  The company is deciding whether to enter specific rating areas in California, where a rating area consists of one or more counties.  Therefore the demographic distribution of the existing market, $p_{X\given M}(x \given 0)$, can be taken to be that of the United States, while the new market distributions $p_{X\given M}(x \given 1)$ correspond to California rating areas.  To simulate these two markets, the MEPS data set is randomly and evenly split into training and test sets.  All results reported in Sections~\ref{sec:empirical:est}--\ref{sec:empirical:privacy} are averaged over $200$ such splits. The existing market distribution $p_{X\given M}(x \given 0)$ is estimated empirically directly from the training set.  The distribution $p_{X\given M}(x \given 1)$ is obtained by reweighing samples from the test set according to the demographics of each rating area, relative to the national baseline represented by MEPS.  Rating area-specific demographics are obtained from the American Community Survey (ACS) \cite{ACS2005}.  

Once the market distributions are created, the enrollment in the company's plan must also be simulated.  We focus on the dependence of enrollment on age.  To generate the existing plan distribution $p_{X\given E,M}(x \given 1,0)$, samples in the existing market data set are reweighed based on the age distribution in the initial enrollment period of the Health Insurance Marketplaces created by the Affordable Care Act \cite{HHS2014,ASPE2014}, again relative to the national baseline.  The resulting distribution differs notably from that of the larger market, $p_{X\given M}(x \given 0)$, in having few children ($<18$) and seniors ($>65$).  The age-dependent enrollment probabilities $p_{E\given X}(1\given x)$ induced by this procedure are then applied to samples in the new market to simulate plan enrollment in the new market, $p_{X\given E,M}(x \given 1,1)$. 

We also consider a second, simpler scenario where the demographic distribution of the existing and the new markets, $p_{X|M}(x | 0)$ and $p_{X|M}(x | 1)$ are both taken to be that of the United States. The MEPS data set is randomly and evenly split into training and test sets, but the test set is not reweighted in this case. We present results for this scenario in Section~\ref{sec:empirical:privacy_noshift} for $200$ such splits. The plan distributions are simulated using the same procedure described above and hence are also equal to each other.


The specific MEPS data set we consider is for the year 2005, containing just over $15000$ weighted records, and the demographic variables 
are gender, age (binned into 8 groups similar to those in \cite{ASPE2014}), education level (0--5), and income level (categories 0--4 relative to the federal poverty level).  The specific cost variable we use is known in MEPS as ``total expenditure'' (TOTEXP) over the year.

\subsection{Reidentification Risk}
\label{sec:empirical:reID}

In this section, we empirically evaluate the reidentification risks of the privacy preservation methods proposed in Section~\ref{sec:privacy}.  Here reidentification risk refers to the probability of correctly retrieving an individual's record from an anonymized data set using their true quasi-identifiers $x$.  It is seen that for a given anonymity parameter $k$, the distribution-recovering transformations discussed in Section~\ref{sec:privacy:trans} have reidentification risks at least as good as or better than standard $k$-anonymization.

Enrollment data for the existing market (corresponding to $p_{X\given E,M}(x\given 1,0)$) is simulated from MEPS as described in Section~\ref{sec:empirical:data}.  Recall that this is the data set that requires privacy protection because it contains health care expenditures.  Each record in the data set is assigned a record number.  The data set then undergoes one of three procedures: 1) standard $k$-anonymization through clustering (Section~\ref{sec:privacy:clustering}), 2) clustering plus distribution recovery via resampling with replacement (Section~\ref{sec:privacy:intra}), and 3) clustering plus distribution recovery via Gaussian dither (Section~\ref{sec:privacy:gaussian}).  The result is an anonymized data set with modified 
quasi-identifiers $\hat{x}$.  Before clustering, all variables $x$ and $y$ are standardized.  The means and variances are restored after the privacy transformations.  We set $w = 1$ in the distortion function \eqref{eq:distortionEuclidean} for clustering and add diagonal loading $\alpha = 1/3$ to the Gaussian dither.

For reidentification, each record in the original data set is matched to the anonymized data set based on quasi-identifiers, $x$ and $\hat{x}$ respectively.  A minimum Euclidean distance criterion is used so that a match is always returned, even if inexact ($\hat{x} \neq x$).  In the 
frequent case of multiple matches 
(multiple $\hat{x}$ at the minimum distance), one of the anonymized records is selected uniformly at random.  A reidentification is declared if the record numbers of the original and matched records are the same.

We first discuss the reidentification performance to be expected under the above scheme.  Successful reidentification requires two events to occur: 1) the original quasi-identifiers $x$ are mapped to a minimum-distance point $\hat{x}$ (often but not necessarily identical to $x$), and 2) the correct record is selected from multiple matches.  Conventional $k$-anonymization focuses on limiting the probability of the second event to no more than $1/k$ by ensuring that every combination of quasi-identifier values (termed an equivalence class) in the anonymized data set is shared by at least $k$ records. The probability can be less than $1/k$ if some equivalence classes have more than $k$ records. In contrast, for the proposed distribution-preserving methods, the distribution over equivalence classes follows that of the original data set and thus cannot be expected to satisfy the $k$-record requirement.  Instead it is the probability of the intersection of the two events that is controlled at the same $1/k$ level.  While an analysis of the general case is not straightforward, for the case of resampling with replacement within a cluster of size $k$ with no values in common with other clusters, it can be verified that the probability of the intersection is exactly $1/k$, as expected from symmetry.  If there are common values, then the probability is generally reduced because the second event of correct selection becomes less likely if a common value is matched.  In the case of Gaussian dither, the first event becomes less likely because there is some probability of mapping $x$ to a more distant point outside the cluster.

With this intuition in mind, we now discuss empirical results obtained by averaging multiple trials of the above anonymization and reidentification procedure.  The randomness in each trial is due to selection from multiple matches in reidentification as well as dither in the distribution-preserving methods.  Since the reidentification algorithm depends on a record only through the equivalence class of $x$, we report reidentification frequencies by equivalence class, of which there are $310$ in the original data set.  

\begin{figure}[h!]
\centerline{
\subfigure[]{
\includegraphics[width=0.50\columnwidth]{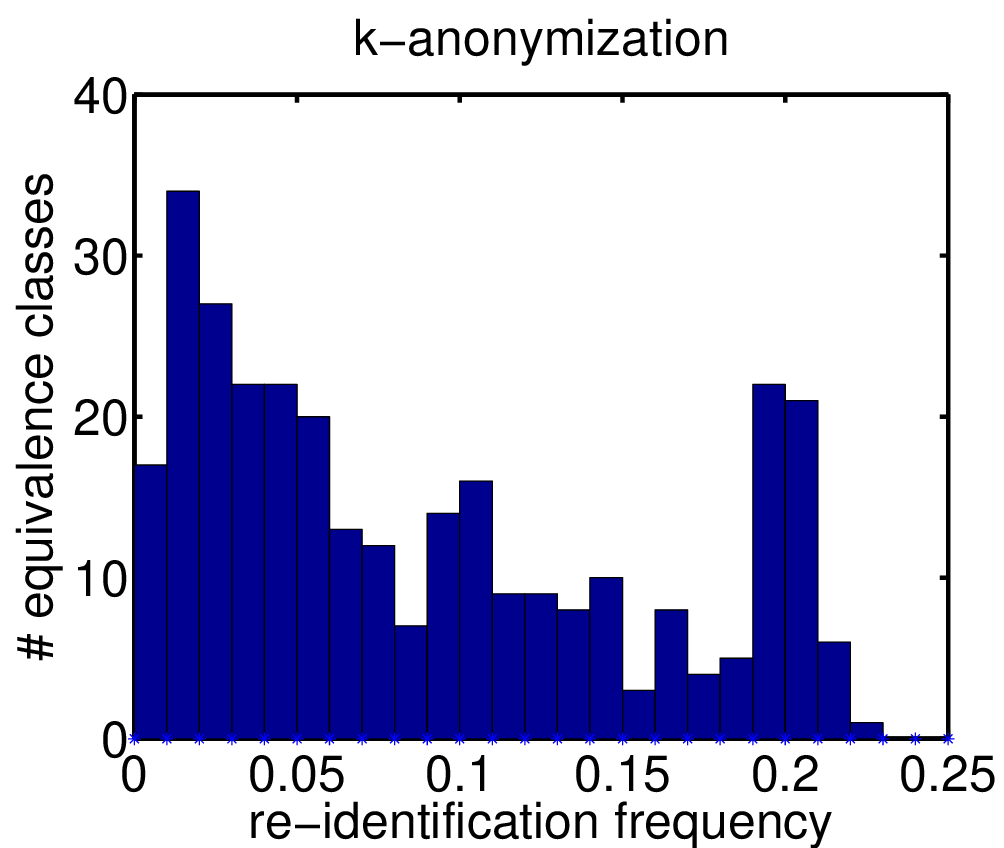}
\label{fig:reID5kanon}}
\subfigure[]{ 
\includegraphics[width=0.50\columnwidth]{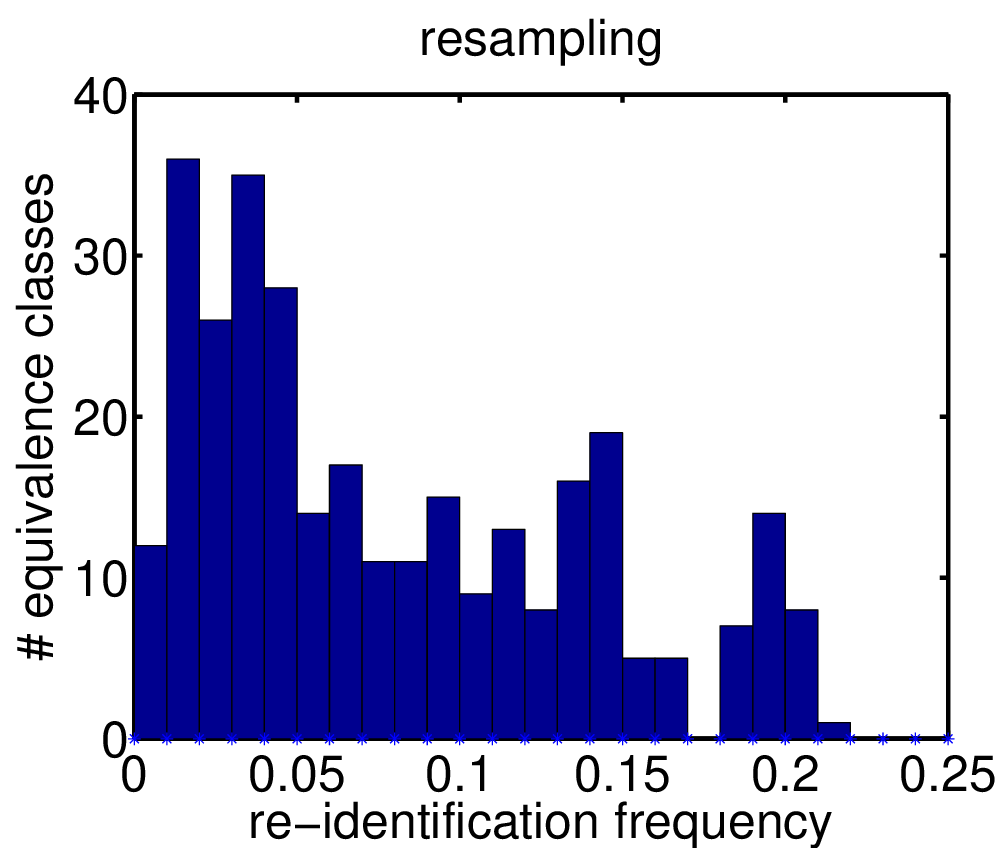}
\label{fig:reID5resamp}}
}
\centerline{
\subfigure[]{ \includegraphics[width=0.50\columnwidth]{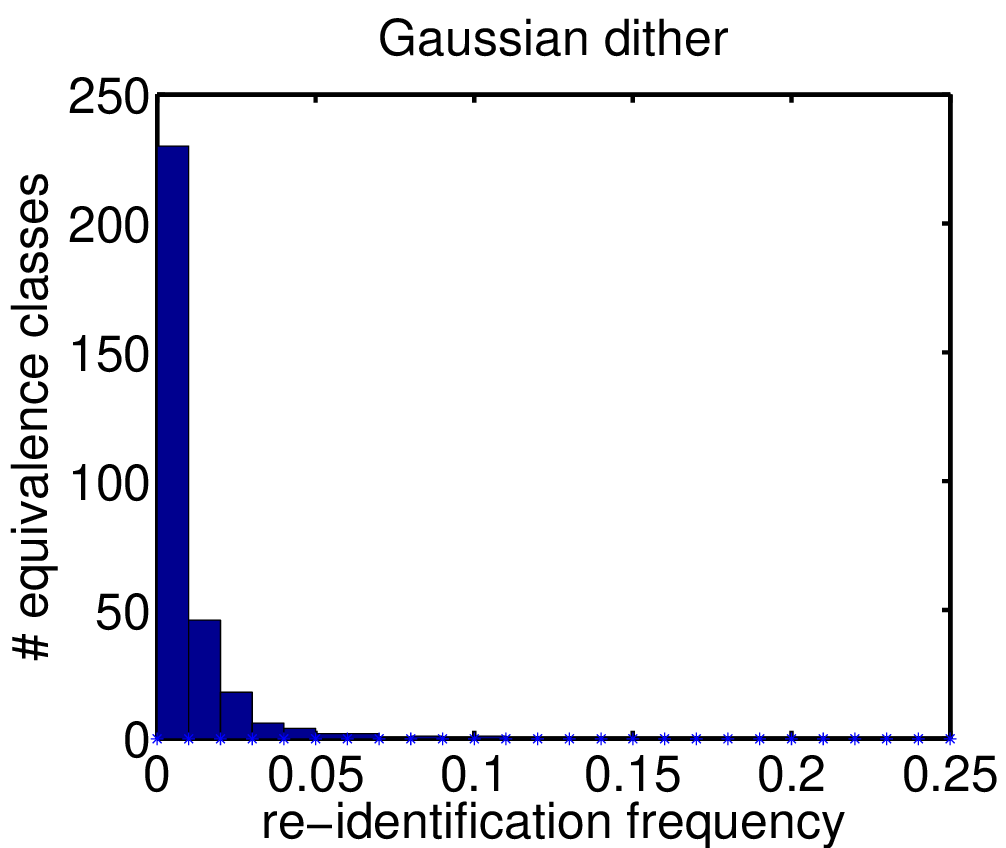}
\label{fig:reID5gauss}}
\subfigure[]{ 
\includegraphics[width=0.50\columnwidth]{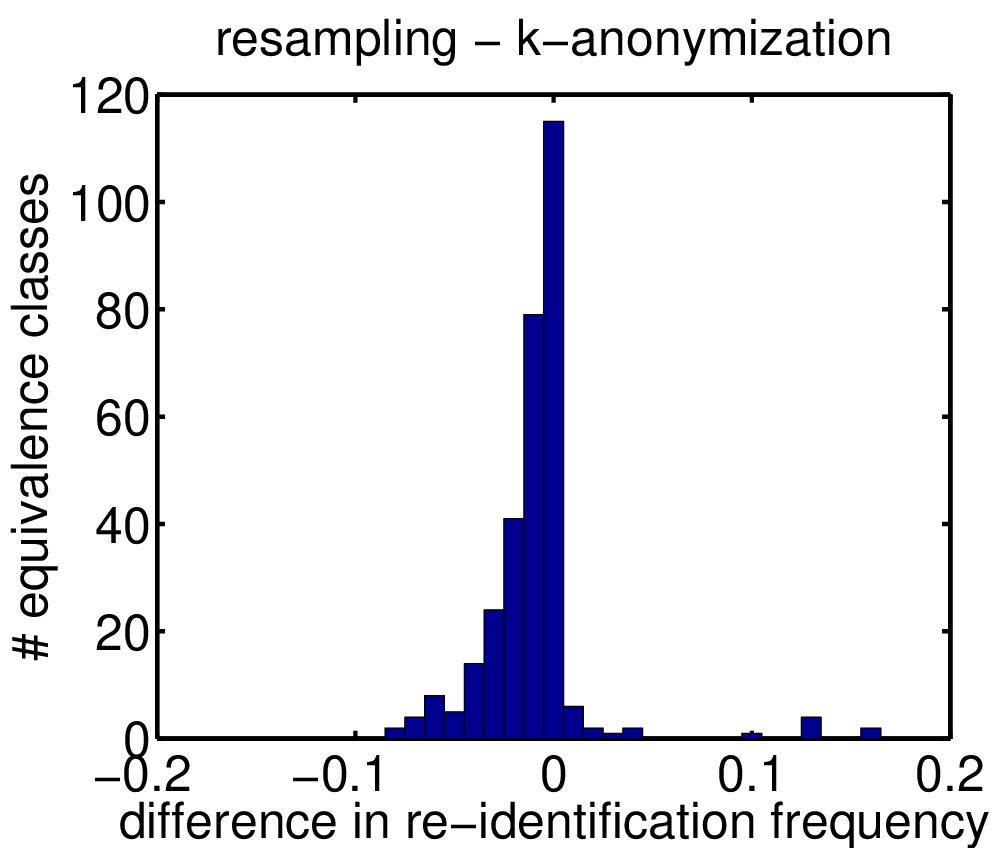}
\label{fig:reID5diff}}
}
\caption{Histograms of reidentification frequencies over $1000$ trials for anonymity parameter $k = 5$.  The proposed distribution-preserving method using (b) resampling with replacement has similar reidentification risk as (a) standard $k$-anonymization, while the distribution-preserving method using (c) Gaussian dither has substantially lower risk.  In (d), a histogram of differences in reidentification frequencies between the resampling method and $k$-anonymization indicates that the former is better for a large majority of equivalence classes.}
\label{fig:reID5}
\end{figure}
Figure~\ref{fig:reID5} shows histograms of reidentification frequencies over equivalence classes for the three privacy protection methods, $k = 5$, and $1000$ trials.  For standard $k$-anonymization in Figure~\ref{fig:reID5kanon}, the reidentification frequency ranges from $0$ to approximately $1/k = 1/5$. As explained above, the distribution reflects the fact that many equivalence classes in the original data set already exceed $5$ members, in some cases by a large factor. Similarly large equivalence classes are maintained after clustering.  The few reidentification frequencies greater than $1/5$ can be attributed to sampling error.  In the worst case, the reidentification frequency is proportional to a binomial random variable with parameters $1000$ and $1/5$, and hence the observed frequency may exceed 
$1/5$.  

For the distribution-preserving method using resampling, Figure~\ref{fig:reID5resamp} shows that the reidentification risk is similar to that of $k$-anonymization.  The histogram depends again on the range of equivalence class sizes and also on the prevalence of common values between clusters, as mentioned earlier.  On the other hand, the Gaussian dither method in Figure~\ref{fig:reID5gauss} has much lower reidentification risk. This can be attributed to the choice of diagonal loading $\alpha = 1/3$, which causes the Gaussian dither to frequently map original quasi-identifiers $x$ to points $\hat{x}$ that are beyond minimum distance. Reducing $\alpha$ would increase the reidentification risk.  Figure~\ref{fig:reID5diff} provides a more detailed comparison between the resampling method and $k$-anonymization by plotting the histogram of differences in reidentification frequencies.  For a large majority of equivalence classes and $k = 5$, resampling offers slightly better privacy protection than $k$-anonymization.  This is partly offset by worse protection for a few equivalence classes.  Figure~\ref{fig:reID10} shows that all of the above patterns hold for $k = 10$ and $2000$ trials.

\begin{figure}[h!]
\centerline{
\subfigure[]{ \includegraphics[width=0.50\columnwidth]{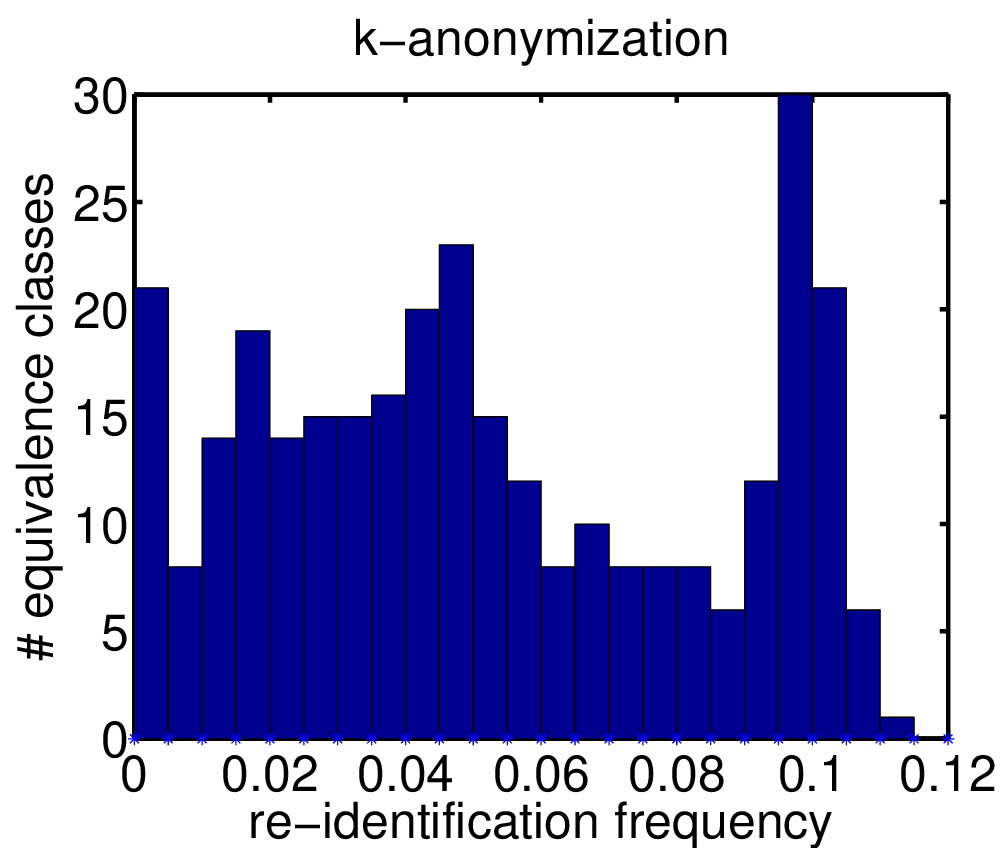}
\label{fig:reID10kanon}}
\subfigure[]{
\includegraphics[width=0.50\columnwidth]{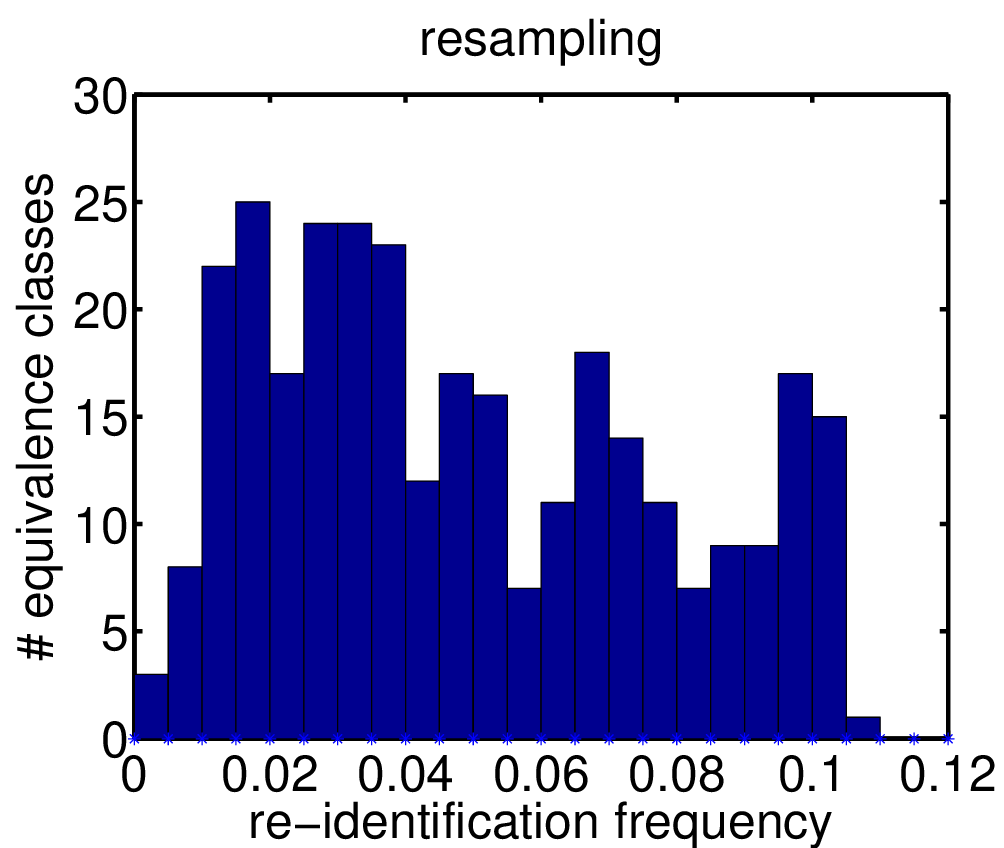}
\label{fig:reID10resamp}}
}
\centerline{
\subfigure[]{ \includegraphics[width=0.50\columnwidth]{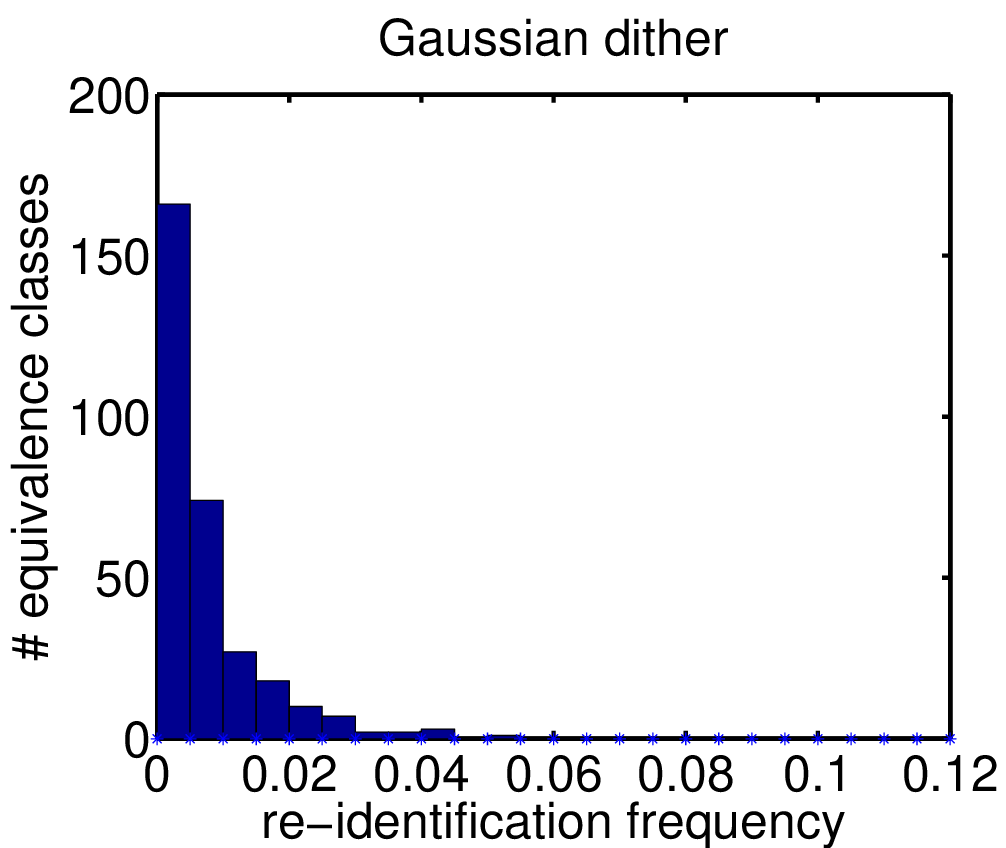}
\label{fig:reID10gauss}}
\subfigure[]{ 
\includegraphics[width=0.50\columnwidth]{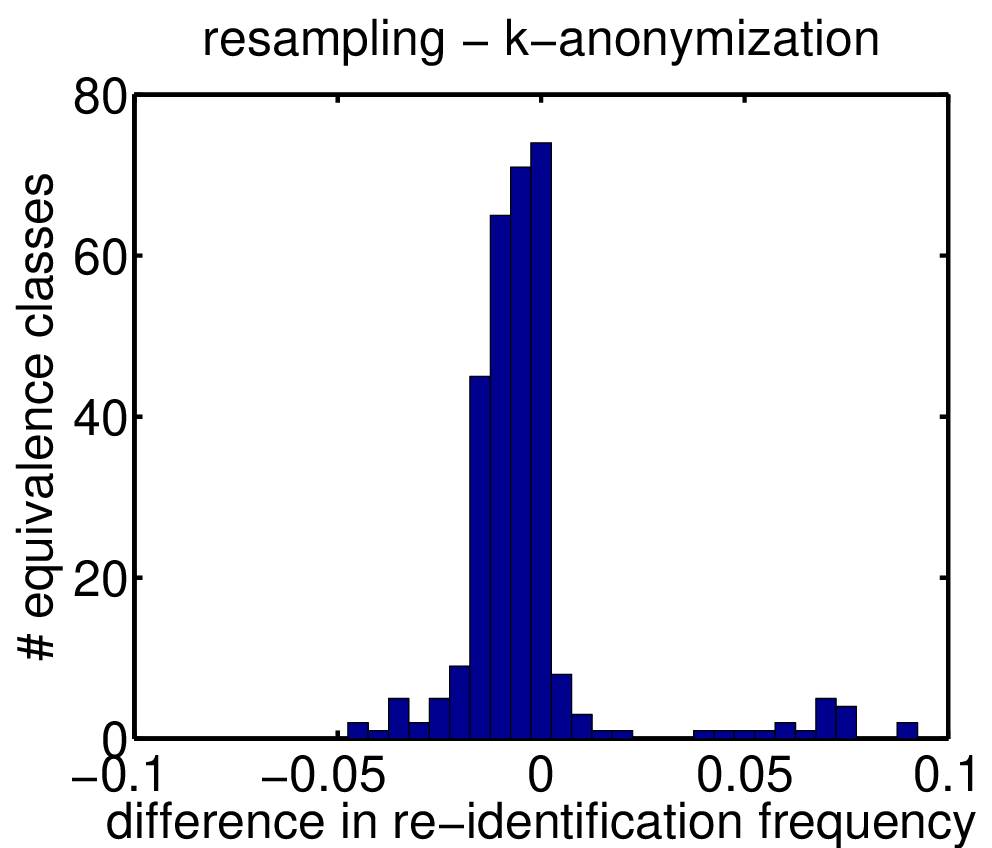}
\label{fig:reID10diff}}
}
\caption{Histograms of reidentification frequencies over $2000$ trials for anonymity parameter $k = 10$.}
\label{fig:reID10}
\end{figure}

For larger $k$, the reidentification probabilities decrease accordingly, and direct empirical validation on a per-equivalence-class basis would require increasing numbers of trials to combat the binomial sampling error mentioned above.  As an alternative, in Figure~\ref{fig:reIDk} we summarize performance by plotting the average reidentification frequency over all records and $200$ trials as a function of $k$.  In all cases, these average reidentification frequencies fall below the nominal value of $1/k$ due to the respective reasons given earlier in this section. The margin is most substantial for $k < 50$ because of large equivalence classes. For smaller $k$, the resampling method happens to yield slightly better privacy than $k$-anonymization while the Gaussian dither method is much stronger, in agreement with Figs.~\ref{fig:reID5} and \ref{fig:reID10}.  As $k$ increases, $k$-anonymization and the resampling method exchange places while the advantage of the Gaussian method diminishes.

\begin{figure}[h!]
\centering
\includegraphics[width=0.70\columnwidth]{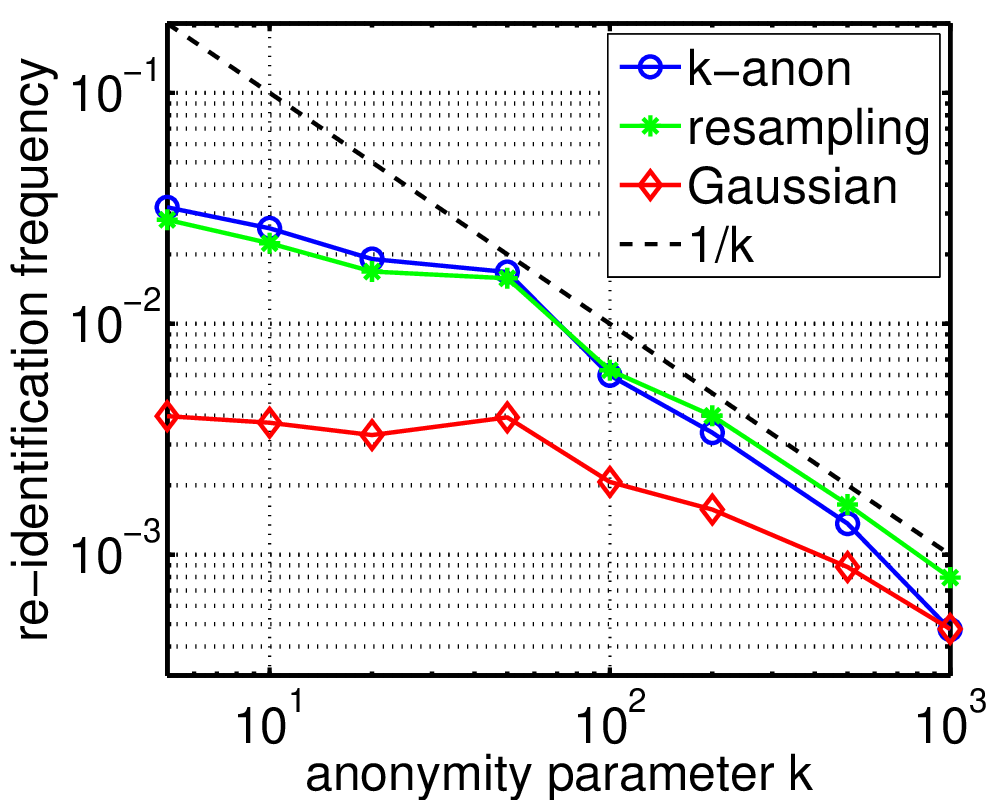}
\caption{Reidentification frequency, averaged over all records and $200$ trials, for different privacy protection methods as a function of anonymity parameter $k$.}
\label{fig:reIDk}
\end{figure}

\subsection{Privacy Preservation and Market Shift}
\label{sec:empirical:privacy_market}
We now discuss results in predicting health care expenditure as a function of the demographic variables age, gender, education level, and income level, as detailed in Section~\ref{sec:empirical:data}. Age, education, and income level are ordinal variables while gender is binary. 
The regression model used in all cases is a sum of univariate functions of each demographic variable. We consider two 
cases as follows.   
\begin{enumerate}
\item The univariate functions 
are not constrained to be linear and can vary arbitrarily with input value.  We refer to this case as ``dummy-coded'' since it can be 
achieved by dummy-coding the discrete variables, i.e., introducing a binary variable for each level while leaving one level out to avoid linear dependence on the intercept term. 
\item The univariate functions are 
constrained to be linear in each variable.  We refer to this second case as ``numeric'' since it 
is achieved by interpreting the levels (e.g.\ 0--7 for age, 0--5 for education) as 
real numeric values. 
\end{enumerate}
The regression model in the second case has many fewer 
degrees of freedom than the more flexible model of the first case.  For this specific MEPS data set, it can be observed that the first model does not offer improvements in prediction accuracy compared to the simpler second model.  We include results for the first model to illustrate the behavior of a commonly used, more complicated model when subjected to privacy preservation and covariate shift.

\begin{figure}[h!]
\centerline{
\subfigure[]{\includegraphics[width=0.50\columnwidth]{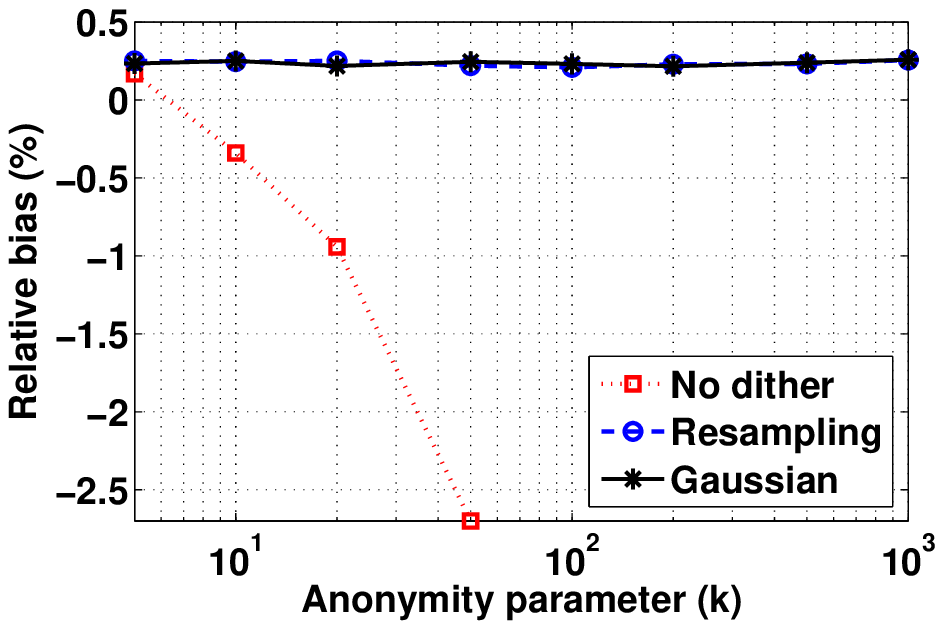}\label{fig:relBias_noshift}}
\subfigure[]{\includegraphics[width=0.52\columnwidth]{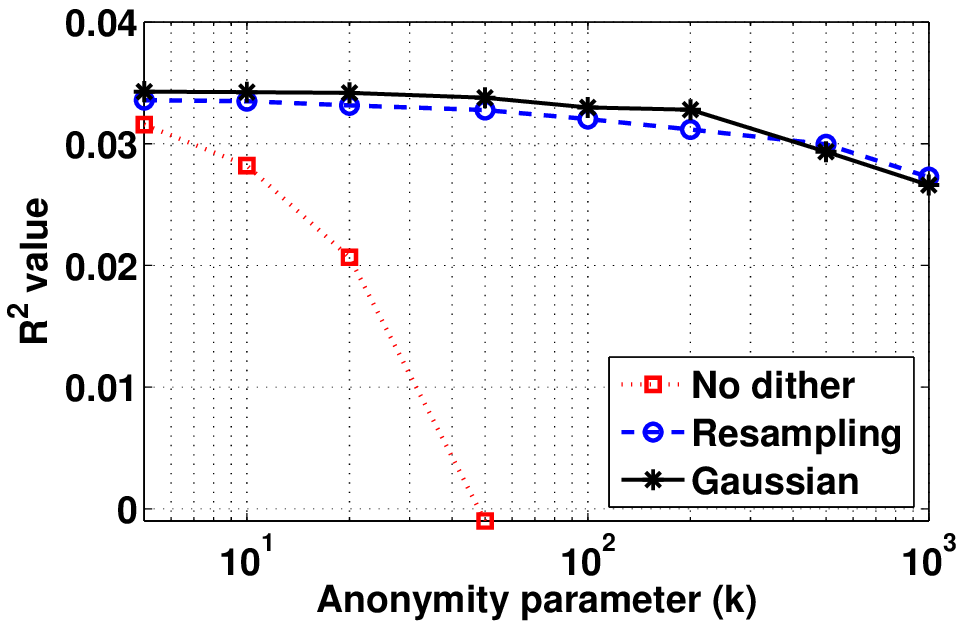}\label{fig:rsq_noshift}}
}
\centerline{
\subfigure[]{\includegraphics[width=0.50\columnwidth]{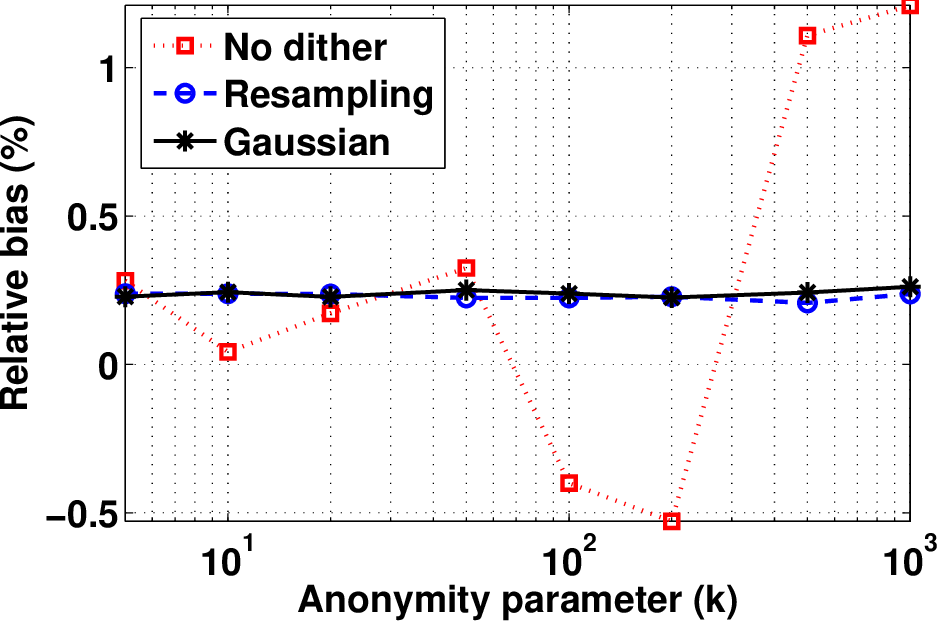}\label{fig:relBias_noshift_nodumm}}
\subfigure[]{\includegraphics[width=0.53\columnwidth]{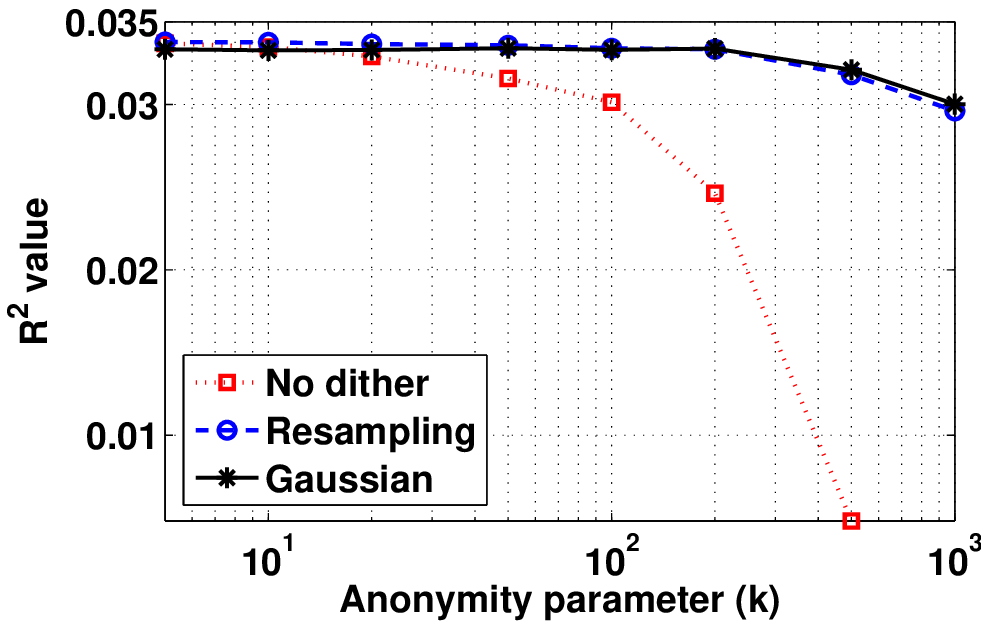}\label{fig:rsq_noshift_nodumm}}
}
\centerline{
\subfigure[]{\includegraphics[width=0.5\columnwidth]{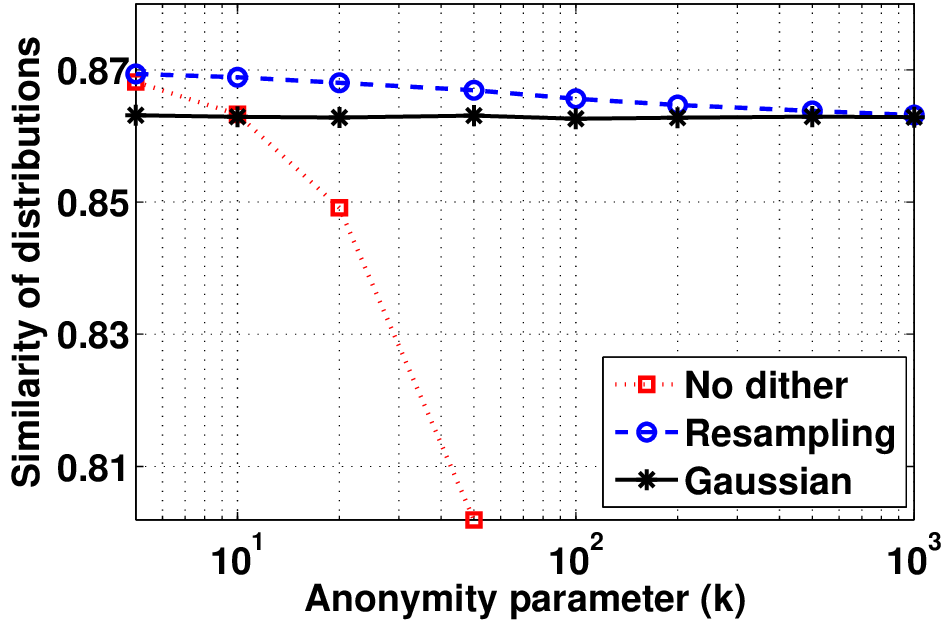}\label{fig:sim_noshift}}
}
\caption{(a) Prediction bias for dummy-coded variables, (b) $R^2$ coefficient for dummy-coded variables, (c) prediction bias for numeric variables, (d) $R^2$ coefficient for numeric variables, and (e) histogram similarity of distributions resulting from $k$-anonymization 
with and without the proposed distribution-preserving procedures. For $k$-anonymity without distribution preservation, relative bias in (a), $R^2$ in (b) and (d), and distribution similarity in (e) all drop drastically with $k$ and reach minimum values of $-53.2$, $-1.193$, $-0.0139$, and $0.1709$ respectively (not shown) for $k = 1000$.  In (c), the lack of distribution preservation makes the bias oscillate but it remains relatively small for all values of $k$.  In contrast, both distribution-preserving procedures result in 
almost constant bias and only 
small drops in $R^2$ and distribution similarity.
}
\label{fig:noshift}
\end{figure}

\subsubsection{Prediction with Privacy Preservation}
\label{sec:empirical:privacy_noshift}

We first present results for the simpler scenario where there is no shift in the demographic distributions between the existing and the new markets, as well as between the corresponding plan enrollment distributions. The plan enrollment data from 
the existing market, i.e.\ the training data, is 
subjected to the same three 
\clearpage
\noindent types of privacy-preserving transformations as 
in Section~\ref{sec:empirical:reID}. 
Figure \ref{fig:noshift} shows the relative bias and coefficient of determination $R^2$ in predicting health care cost in the new market as the anonymity parameter $k$ increases.  For the nonlinear dummy-coded model in Figure~\ref{fig:noshift}(a) and Figure \ref{fig:noshift}(b), 
it is clear that 
performance drops dramatically 
when distribution-preserving approaches are not used. 
Conversely, performance remains either constant or only drops slightly using the proposed distribution-preserving approaches.  For the linear numeric model in Figure~\ref{fig:noshift}(c) and Figure \ref{fig:noshift}(d), the difference is less stark: Without distribution preservation, relative bias oscillates at a relatively low level while $R^2$ decreases to a negative but less extreme value.  These results suggest that the simpler linear model is inherently more robust to privacy transformations but can nevertheless be improved by the proposed distribution-preservation methods. 

The poor performance of $k$-anonymization without distribution preservation can be understood from two perspectives.  First, since in conventional $k$-anonymization, samples of $X$ are replaced with the corresponding cluster centers $\bar{x}_{\ell}$, the number of unique samples for training the regression model decreases to $\left \lfloor {\frac{n}{k}} \right \rfloor$, a substantial reduction when $k$ is large.  It is well-known that the prediction error of a linear regression model tends to decrease at a $O(1/\sqrt{n})$ rate, where $n$ is the training sample size \cite{Shao2003}. 
Second, in Figure~\ref{fig:noshift}(e) we plot the similarity between the training data distribution after privacy transformation and the test data distribution.  We use the histogram intersection similarity \cite{swain1991color} for concreteness, for which a value close to $1$ implies that the predictor is trained on a distribution much like the one encountered in testing.  Figure~\ref{fig:noshift}(e) indicates that this similarity falls dramatically with $k$ under conventional $k$-anonymization. 
Overall, it is 
clear that 
distribution preservation is beneficial 
to achieving useful 
$k$-anonymity even when simple workloads such as linear regression are involved.


\begin{table*}[h!]
  \centering
  \caption{Dummy-coded variable performance results for rating areas of California as new markets for no shift, logistic shift, and non-parametric shift.}
  \small
    \begin{tabular}{r|rrr|rrr}
    \hline
    New Market & \multicolumn{3}{|c}{$R^2$ value} & \multicolumn{3}{|c}{Relative Bias (\%)} \\
    \hline
    \multicolumn{1}{c|}{} & No Shift & Logistic & Non-param. & No Shift & Logistic & Non-param. \\
    \multicolumn{1}{c}{RA 1} & \multicolumn{1}{|c}{0.0300} & \multicolumn{1}{c}{0.0271} & \multicolumn{1}{c}{0.0280} & \multicolumn{1}{|c}{4.38} & \multicolumn{1}{c}{4.64} & \multicolumn{1}{c}{1.86} \\
    
    \multicolumn{1}{c}{RA 2} & \multicolumn{1}{|c}{0.0243} & \multicolumn{1}{c}{0.0228} & \multicolumn{1}{c}{0.0220} & \multicolumn{1}{|c}{3.75} & \multicolumn{1}{c}{2.14} & \multicolumn{1}{c}{-0.49} \\
    
    \multicolumn{1}{c}{RA 3} & \multicolumn{1}{|c}{0.0289} & \multicolumn{1}{c}{0.0272} & \multicolumn{1}{c}{0.0270} & \multicolumn{1}{|c}{2.34} & \multicolumn{1}{c}{1.87} & \multicolumn{1}{c}{-0.24} \\
    
    \multicolumn{1}{c}{RA 4} & \multicolumn{1}{|c}{0.0291} & \multicolumn{1}{c}{0.0278} & \multicolumn{1}{c}{0.0264} & \multicolumn{1}{|c}{-2.54} & \multicolumn{1}{c}{-2.39} & \multicolumn{1}{c}{1.40} \\
    
    \multicolumn{1}{c}{RA 5} & \multicolumn{1}{|c}{0.0221} & \multicolumn{1}{c}{0.0212} & \multicolumn{1}{c}{0.0214} & \multicolumn{1}{|c}{3.63} & \multicolumn{1}{c}{1.76} & \multicolumn{1}{c}{0.53} \\
    
    \multicolumn{1}{c}{RA 6} & \multicolumn{1}{|c}{0.0235} & \multicolumn{1}{c}{0.0232} & \multicolumn{1}{c}{0.0227} & \multicolumn{1}{|c}{2.17} & \multicolumn{1}{c}{1.42} & \multicolumn{1}{c}{0.69} \\
    
    \multicolumn{1}{c}{RA 7} & \multicolumn{1}{|c}{0.0233} & \multicolumn{1}{c}{0.0227} & \multicolumn{1}{c}{0.0219} & \multicolumn{1}{|c}{2.31} & \multicolumn{1}{c}{1.40} & \multicolumn{1}{c}{-0.39} \\
    
    \multicolumn{1}{c}{RA 8} & \multicolumn{1}{|c}{0.0213} & \multicolumn{1}{c}{0.0213} & \multicolumn{1}{c}{0.0196} & \multicolumn{1}{|c}{1.12} & \multicolumn{1}{c}{-0.69} & \multicolumn{1}{c}{-1.29} \\
    
    \multicolumn{1}{c}{RA 9} & \multicolumn{1}{|c}{0.0233} & \multicolumn{1}{c}{0.0219} & \multicolumn{1}{c}{0.0230} & \multicolumn{1}{|c}{5.99} & \multicolumn{1}{c}{5.67} & \multicolumn{1}{c}{0.32} \\
    
    \multicolumn{1}{c}{RA 10} & \multicolumn{1}{|c}{0.0330} & \multicolumn{1}{c}{0.0316} & \multicolumn{1}{c}{0.0323} & \multicolumn{1}{|c}{3.48} & \multicolumn{1}{c}{2.87} & \multicolumn{1}{c}{0.57} \\
    
    \multicolumn{1}{c}{RA 11} & \multicolumn{1}{|c}{0.0314} & \multicolumn{1}{c}{0.0290} & \multicolumn{1}{c}{0.0306} & \multicolumn{1}{|c}{4.90} & \multicolumn{1}{c}{4.87} & \multicolumn{1}{c}{2.17} \\
    
    \multicolumn{1}{c}{RA 12} & \multicolumn{1}{|c}{0.0246} & \multicolumn{1}{c}{0.0233} & \multicolumn{1}{c}{0.0243} & \multicolumn{1}{|c}{4.29} & \multicolumn{1}{c}{3.49} & \multicolumn{1}{c}{0.18} \\
    
    \multicolumn{1}{c}{RA 13} & \multicolumn{1}{|c}{0.0328} & \multicolumn{1}{c}{0.0294} & \multicolumn{1}{c}{0.0253} & \multicolumn{1}{|c}{-1.03} & \multicolumn{1}{c}{-0.53} & \multicolumn{1}{c}{0.10} \\
    
    \multicolumn{1}{c}{RA 14} & \multicolumn{1}{|c}{0.0345} & \multicolumn{1}{c}{0.0331} & \multicolumn{1}{c}{0.0330} & \multicolumn{1}{|c}{2.47} & \multicolumn{1}{c}{1.79} & \multicolumn{1}{c}{0.36} \\
    
    \multicolumn{1}{c}{RA 15-16} & \multicolumn{1}{|c}{0.0295} & \multicolumn{1}{c}{0.0281} & \multicolumn{1}{c}{0.0286} & \multicolumn{1}{|c}{2.83} & \multicolumn{1}{c}{2.62} & \multicolumn{1}{c}{0.81} \\
    
    \multicolumn{1}{c}{RA 17} & \multicolumn{1}{|c}{0.0318} & \multicolumn{1}{c}{0.0300} & \multicolumn{1}{c}{0.0302} & \multicolumn{1}{|c}{0.72} & \multicolumn{1}{c}{0.24} & \multicolumn{1}{c}{-0.77} \\
    
    \multicolumn{1}{c}{RA 18} & \multicolumn{1}{|c}{0.0250} & \multicolumn{1}{c}{0.0246} & \multicolumn{1}{c}{0.0252} & \multicolumn{1}{|c}{3.75} & \multicolumn{1}{c}{2.85} & \multicolumn{1}{c}{0.84} \\
    
    \multicolumn{1}{c}{RA 19} & \multicolumn{1}{|c}{0.0268} & \multicolumn{1}{c}{0.0258} & \multicolumn{1}{c}{0.0268} & \multicolumn{1}{|c}{3.36} & \multicolumn{1}{c}{3.07} & \multicolumn{1}{c}{1.23} \\    
    \hline
    \end{tabular}%
  \label{tab:empres_nopriv}%
\end{table*}%

\subsubsection{Prediction 
with Market Shift}
\label{sec:empirical:est}
Next we discuss cost prediction in the absence of privacy-preserving data transformations but with 
a shift in demographic distribution between 
the existing and the new markets.  Two covariate shift methods are compared, corresponding respectively to non-parametric and logistic regression methods of estimating the importance weights $w(x)$. 
A baseline method that does not account for covariate shift is also compared.  

Table \ref{tab:empres_nopriv} and Table \ref{tab:empres_nopriv_nodumm} summarize performance in predicting the cost in 
rating areas of California as new markets using nonlinear dummy-coded and linear numeric models respectively. As discussed in Section~\ref{sec:mra}, 
for aggregate prediction the bias is often the more important performance metric.  
Table \ref{tab:empres_nopriv} shows that the baseline approach using the dummy-coded model has a noticeable bias, and the covariate shift approaches reduce the bias for most new markets by shifting the distribution of existing plan members to look more like prospective enrollees in the new market. This reduction is particularly significant with the non-parametric shift method. With the linear numeric model in Table \ref{tab:empres_nopriv_nodumm}, the relative bias is lower for the baseline, but still, the non-parametric shift method tends to offer a reduction over the baseline approach in most rating areas. In fact, the biases of the covariate shift methods under the dummy-coded model and the numeric model are similar in most markets. Performing covariate shift is more useful with the nonlinear dummy-coded model, which is very common in several applications, but also provides some advantage for the numeric model.

\begin{table}[h!]
  \centering
  \caption{Numeric variable performance results for rating areas of California as new markets for no shift, logistic shift, and non-parametric shift.}
    \small
    \begin{tabular}{c|ccc|ccc}
    \hline
    New Market & \multicolumn{3}{|c}{$R^2$ value} & \multicolumn{3}{|c}{Relative Bias (\%)} \\
    \hline
   & No Shift & Logistic & Non-param. & No Shift & Logistic & Non-param. \\
    RA 1  & 0.0312 & 0.0286 & 0.0291 & 0.47  & 4.49  & 1.81 \\
    RA 2  & 0.0260 & 0.0247 & 0.0242 & 1.15  & 1.76  & -0.48 \\
    RA 3  & 0.0288 & 0.0275 & 0.0273 & 0.06  & 1.56  & -0.30 \\
    RA 4  & 0.0296 & 0.0289 & 0.0279 & -2.39 & -1.94 & 1.27 \\
    RA 5  & 0.0247 & 0.0238 & 0.0238 & 0.97  & 1.47  & 0.48 \\
    RA 6  & 0.0252 & 0.0249 & 0.0248 & 0.22  & 1.43  & 0.63 \\
    RA 7  & 0.0260 & 0.0257 & 0.0256 & 0.90  & 1.65  & -0.45 \\
    RA 8  & 0.0245 & 0.0239 & 0.0232 & -1.20 & -0.83 & -1.34 \\
    RA 9  & 0.0236 & 0.0225 & 0.0235 & 4.41  & 5.94  & 0.30 \\
    RA 10 & 0.0316 & 0.0304 & 0.0311 & 3.01  & 3.57  & 0.50 \\
    RA 11 & 0.0308 & 0.0290 & 0.0305 & 3.16  & 5.67  & 2.04 \\
    RA 12 & 0.0265 & 0.0253 & 0.0259 & 1.15  & 3.41  & 0.13 \\
    RA 13 & 0.0326 & 0.0307 & 0.0291 & -0.85 & 0.13  & -0.09 \\
    RA 14 & 0.0337 & 0.0328 & 0.0328 & 1.87  & 2.52  & 0.20 \\
    RA 15-16 & 0.0292 & 0.0287 & 0.0290 & 2.96  & 3.05  & 0.73 \\
    RA 17 & 0.0312 & 0.0301 & 0.0300 & 0.86  & 0.86  & -0.85 \\
    RA 18 & 0.0258 & 0.0252 & 0.0258 & 2.63  & 3.00  & 0.79 \\
    RA 19 & 0.0283 & 0.0272 & 0.0279 & 1.45  & 3.18  & 1.17 \\
    \hline
    \end{tabular}%
  \label{tab:empres_nopriv_nodumm}%
\end{table}

\subsubsection{Prediction 
with Privacy Preservation and Market Shift}
\label{sec:empirical:privacy}
Lastly we present results for the scenario in which 
both the insurer's existing plan data is subjected to the three privacy transformations used 
in Sections~\ref{sec:empirical:reID} and \ref{sec:empirical:privacy_noshift}, 
and there is a demographic shift between the existing and new markets.  We focus in this subsection on California rating area $18$ as the new market. 


Figure~\ref{fig:Gr_none} shows results for $k$-anonymization without distribution preservation. As $k$ increases, the original samples from the plan data are represented more and more coarsely by their cluster centers. As a consequence, covariate shift methods fail and prediction accuracy suffers markedly.  In fact, both bias and $R^2$ are unacceptably bad for the dummy-coded data model.  For the numeric data model, the bias does not actually deteriorate by much under the baseline approach, again pointing toward the robustness of the simple linear model.  However, $R^2$ still decreases to negative values.  The covariate shift methods continue to yield unacceptable bias but manage to keep $R^2$ positive.

Figure~\ref{fig:Gr_Re} shows the prediction performance of the proposed privacy-preserving procedure using resampling with replacement combined with different covariate shift methods.  In great contrast to Figure~\ref{fig:Gr_none}, the relative bias stays low for all values of $k$ while $R^2$ decreases only slightly with increasing $k$.  The main difference between the dummy-coded and numeric models is that the bias remains almost constant in the latter case when the baseline method is used.

Figure~\ref{fig:Gr_Ga} depicts performance under privacy preservation with Gaussian dither.  In this case, the bias remains controlled but increases with $k$ to a level higher than for the resampling method in Figure~\ref{fig:Gr_Re}.  A possible explanation for the difference between the two distribution-preserving methods can be found in Figures~\ref{fig:reID5}--\ref{fig:reIDk} and Section~\ref{sec:empirical:reID}.  There it is seen that the Gaussian dither method has much lower reidentification risk, likely as a result of mapping original quasi-identifiers $x$ to distant points $\hat{x}$.  This mapping however may further distort the relationship between $X$ and $Y$ relative to the original, causing the prediction bias to increase.  
In the dummy-coded case, the non-parametric covariate shift method 
succeeds at reducing bias for all $k$. 
In the numeric case, covariate shift methods reduce bias only for $k < 50$.  
On the other hand, $R^2$ is slightly lower for the covariate shift methods because the reweighting of training samples to reduce bias also introduces some additional variability. 

\begin{figure}[H]
\centerline{
\subfigure[]{\includegraphics[width=0.50\columnwidth]{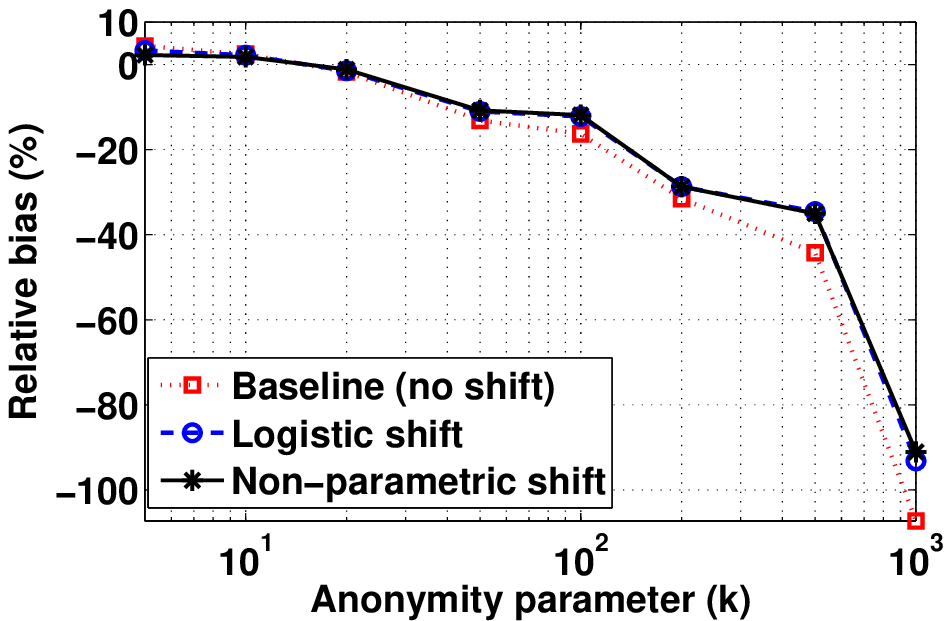}\label{fig:relBias_Gr_none}}
\subfigure[]{\includegraphics[width=0.50\columnwidth]{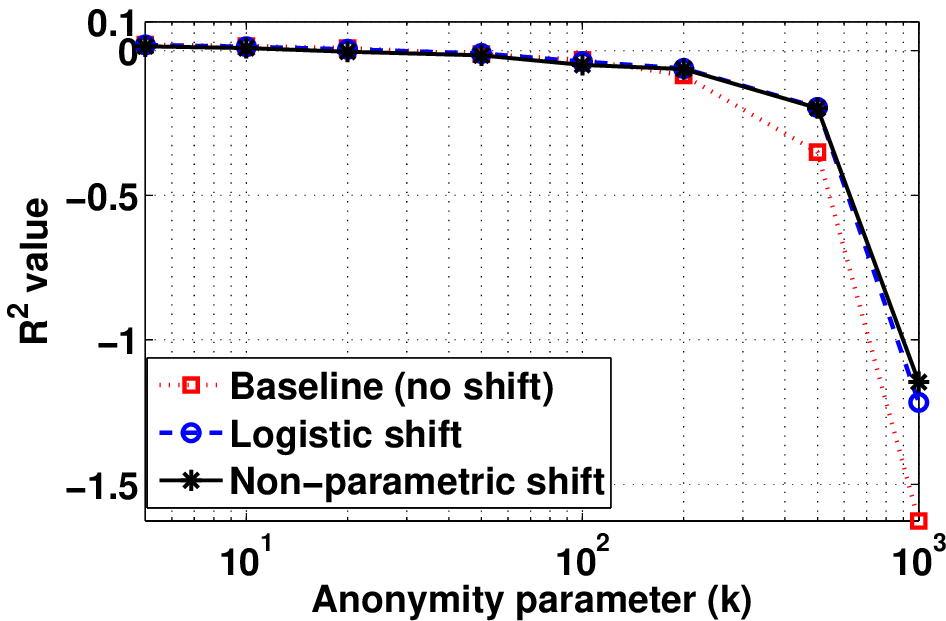}\label{fig:rsq_Gr_none}}
}
\centerline{
\subfigure[]{\includegraphics[width=0.50\columnwidth]{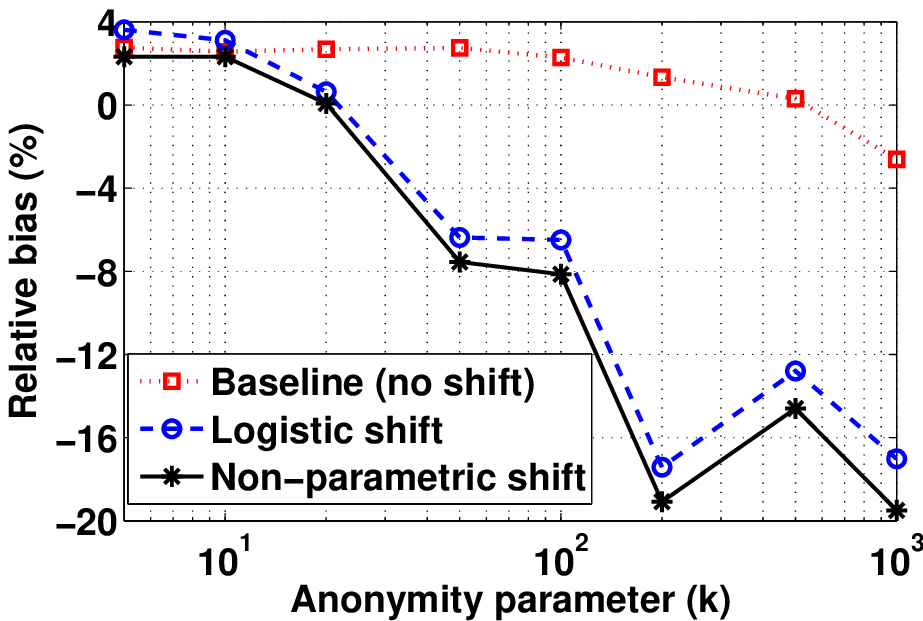}\label{fig:relBias_Gr_none_nodumm}}
\subfigure[]{\includegraphics[width=0.55\columnwidth]{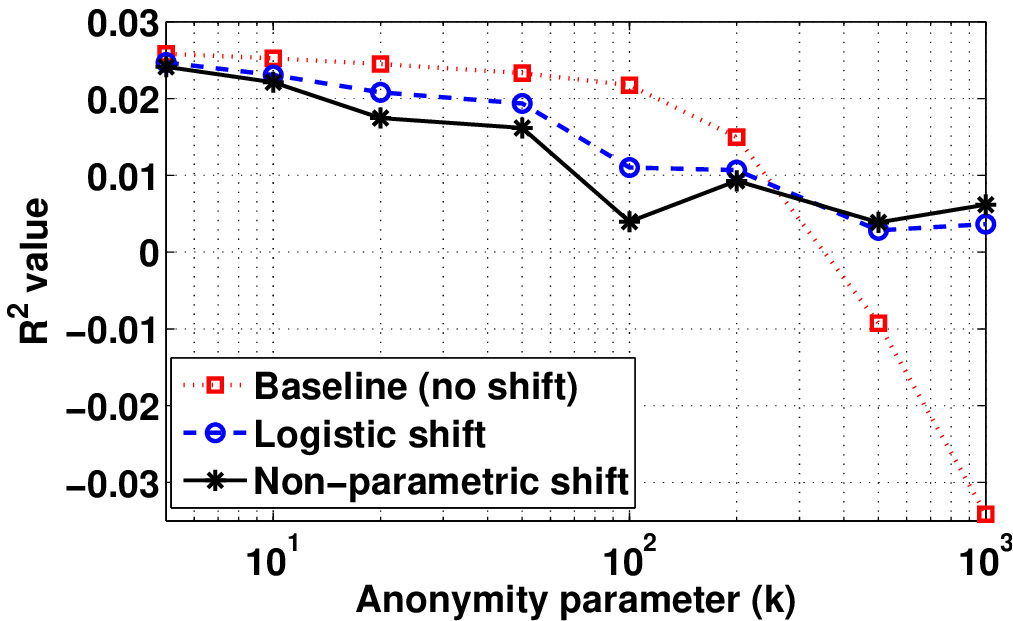}\label{fig:rsq_Gr_none_nodumm}}
}
\caption{Performance results with no distribution preservation: (a) dummy-coded prediction bias, (b) dummy-coded $R^2$ coefficient, (c) numeric prediction bias, and (d) numeric $R^2$ coefficient. For dummy-coded data, the prediction error increases unacceptably as $k$ increases. For numeric data, the bias drops only moderately for the baseline case whereas the $R^2$ coefficient becomes unacceptable for large values of $k$. With the covariate shift methods, the bias drops rapidly, whereas the $R^2$ coefficient decreases slowly, as $k$ increases.}
\label{fig:Gr_none}
\end{figure}

\begin{figure}[H]
\centerline{
\subfigure[]{\includegraphics[width=0.50\columnwidth]{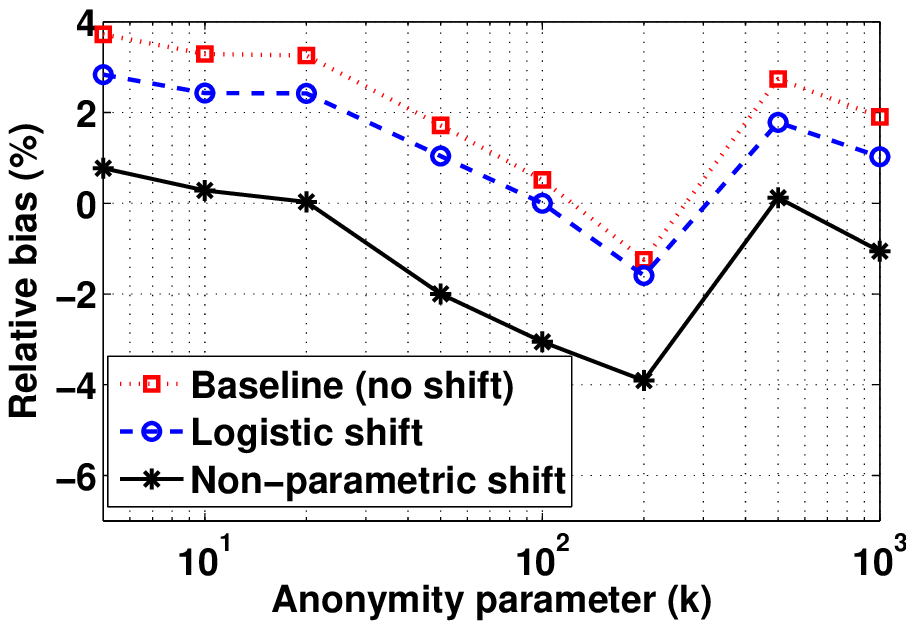}\label{fig:relBias_Gr_Re}}
\subfigure[]{\includegraphics[width=0.525\columnwidth]{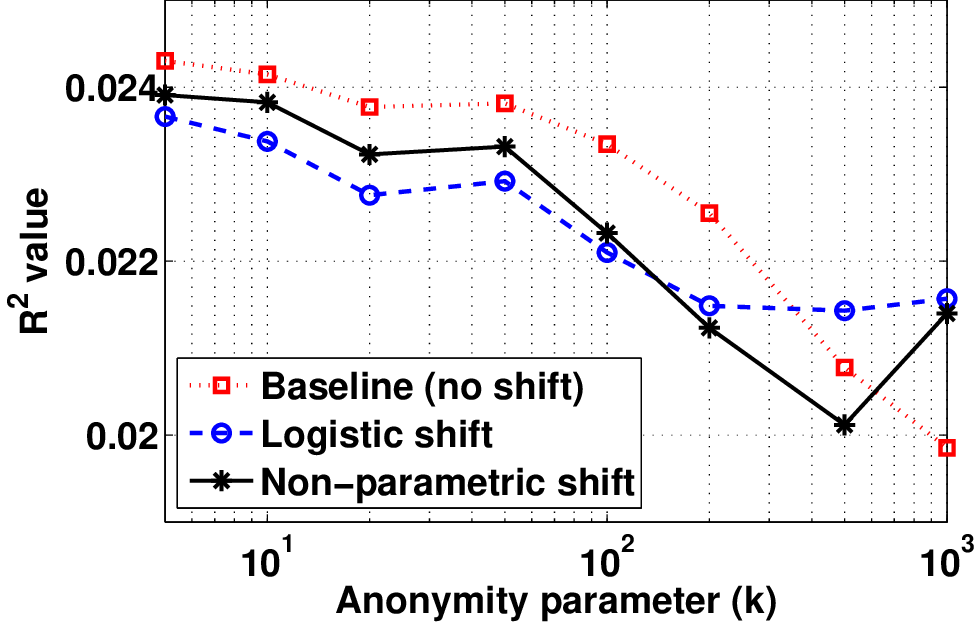}\label{fig:rsq_Gr_Re}}
}
\centerline{
\subfigure[]{\includegraphics[width=0.50\columnwidth]{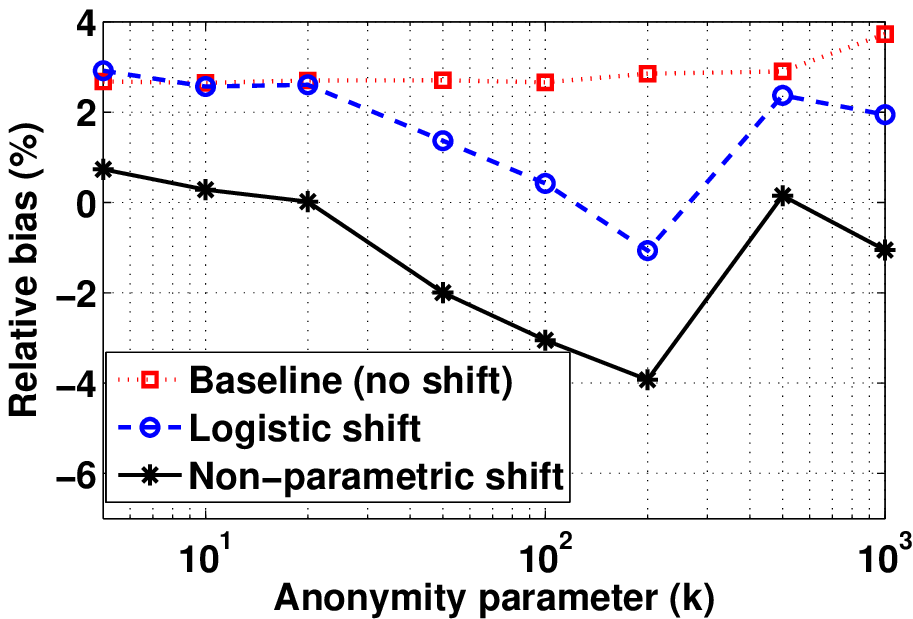}\label{fig:relBias_Gr_Re_nodumm}}
\subfigure[]{\includegraphics[width=0.525\columnwidth]{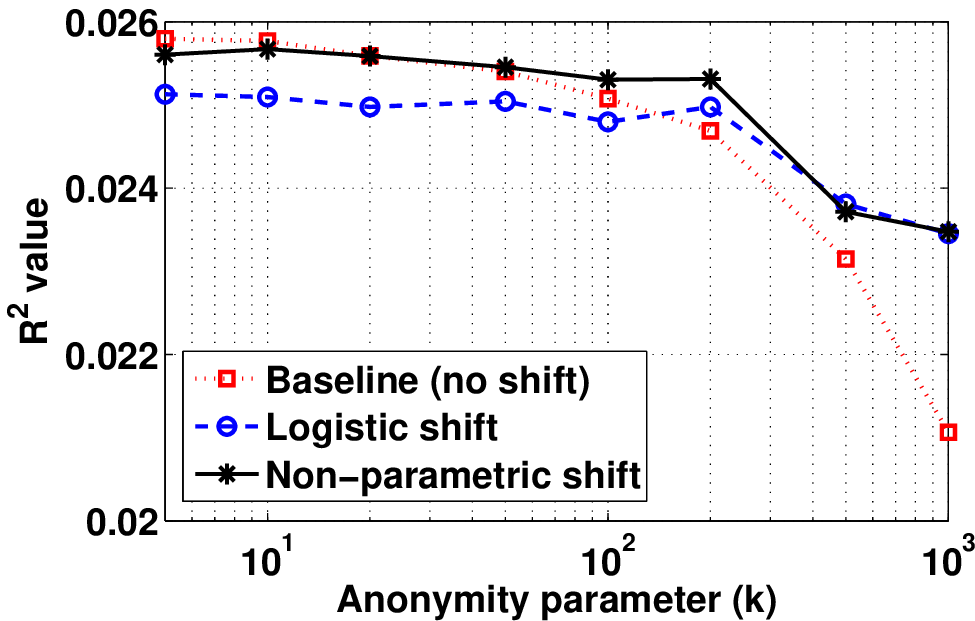}\label{fig:rsq_Gr_Re_nodumm}}
}
\caption{Performance results for intra-cluster dither: (a) dummy-coded prediction bias, (b) dummy-coded $R^2$ coefficient, (c) numeric prediction bias, and (d) numeric $R^2$ coefficient.  For dummy-coded data, as $k$ increases, relative bias stays low whereas the $R^2$ coefficient has a generally decreasing trend. For numeric data, as $k$ increases, the baseline results in an almost-constant bias and a steady decrease in the $R^2$ coefficient. For the two shift methods, the relative bias with numeric data is similar to the case with dummy-coded data.}
\label{fig:Gr_Re}
\end{figure}

\begin{figure}[H]
\centerline{
\subfigure[]{\includegraphics[width=0.50\columnwidth]{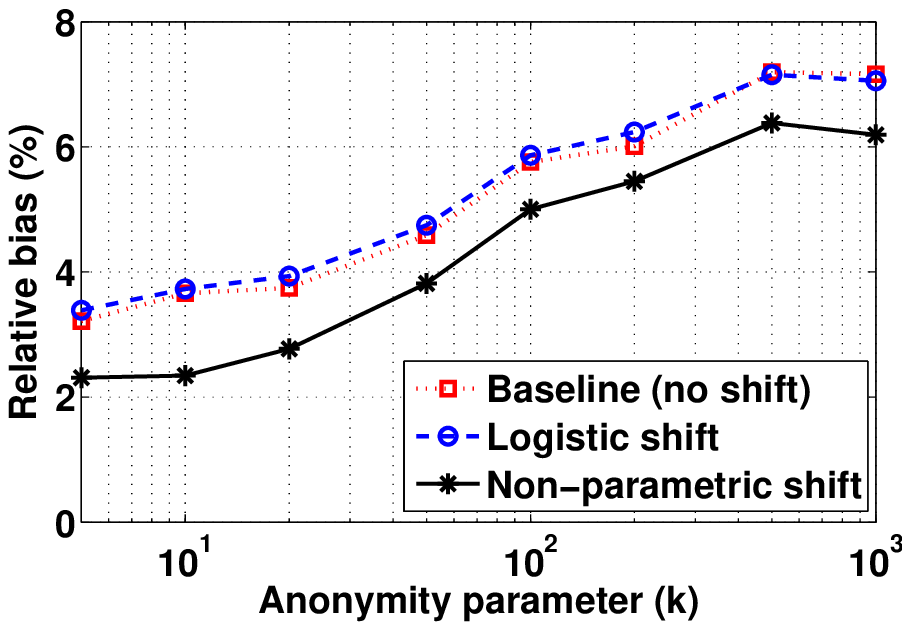}\label{fig:relBias_Gr_Ga}}
\subfigure[]{\includegraphics[width=0.55\columnwidth]{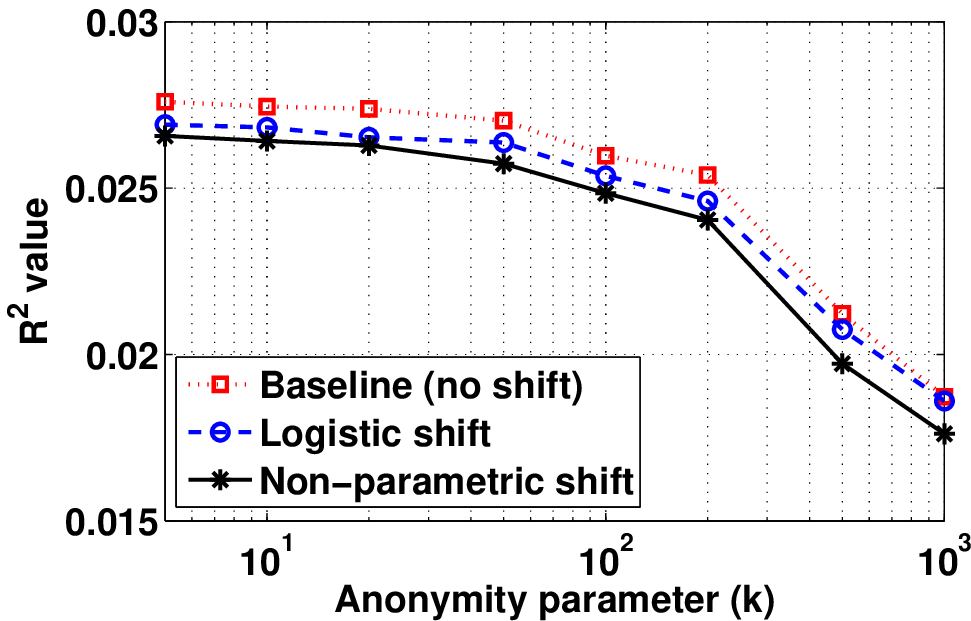}\label{fig:rsq_Gr_Ga}}
}
\centerline{
\subfigure[]{\includegraphics[width=0.50\columnwidth]{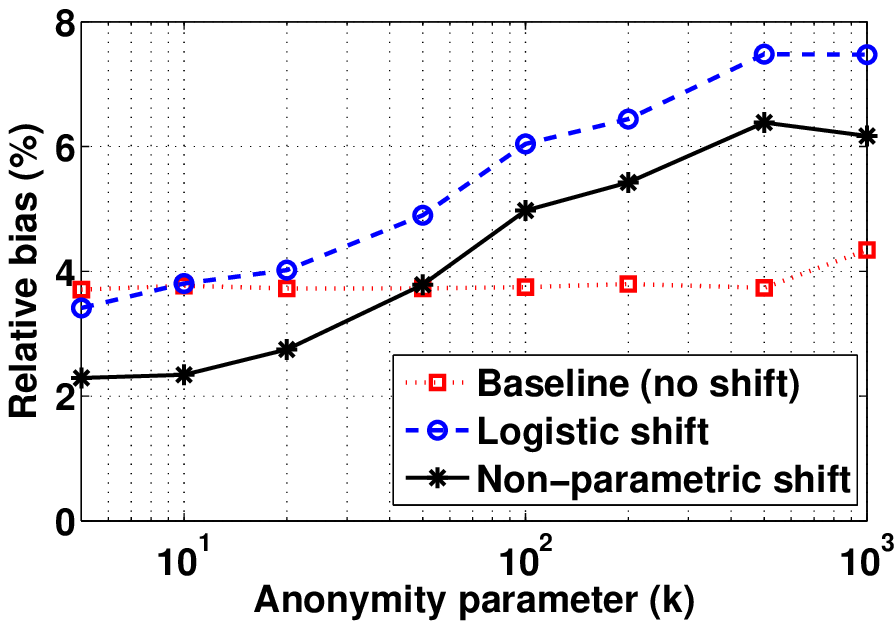}\label{fig:relBias_Gr_Ga_nodumm}}
\subfigure[]{\includegraphics[width=0.55\columnwidth]{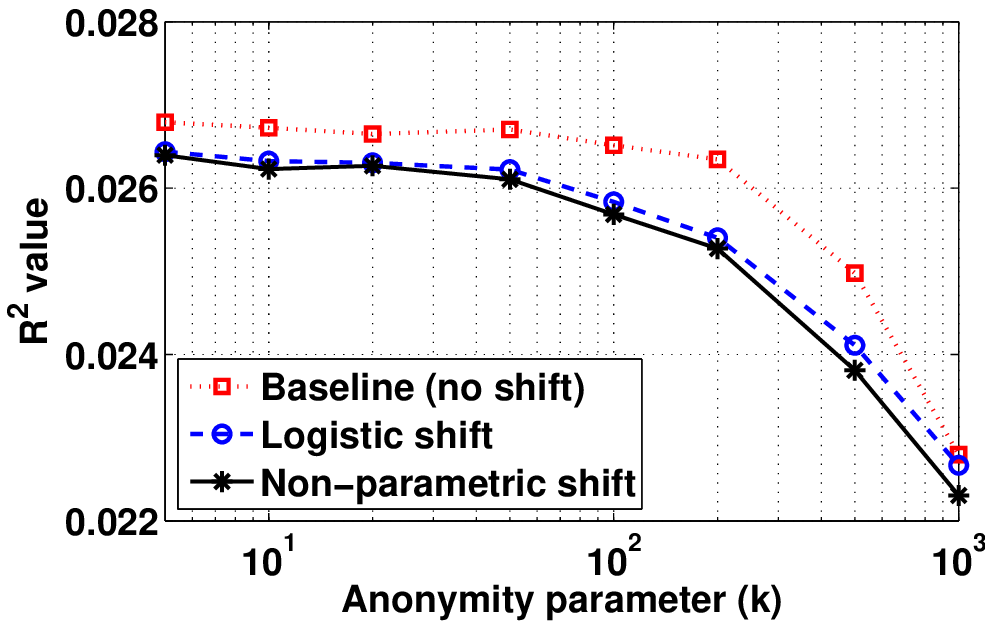}\label{fig:rsq_Gr_Ga_nodumm}}
}
\caption{Performance results for Gaussian dither: (a) dummy-coded prediction bias, (b) dummy-coded $R^2$ coefficient, (c) numeric prediction bias, and (d) numeric $R^2$ coefficient. For dummy-coded data, as $k$ increases, distribution preservation moderates the increase in bias, while the covariate shift methods reduce bias. For numeric data, as $k$ increases, the baseline results in an almost-constant bias and a steady decrease in the $R^2$ coefficient. For the two shift methods, the relative bias with numeric data is similar to the case with dummy-coded data.}
\label{fig:Gr_Ga}
\end{figure}

\begin{figure}[H]
\centerline{
\subfigure[]{\includegraphics[width=0.50\columnwidth]{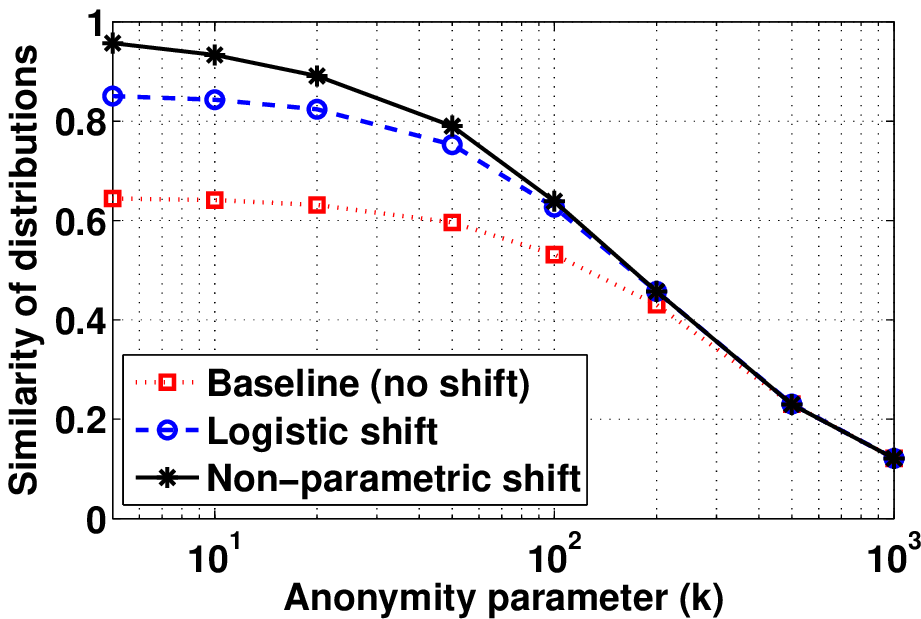}\label{fig:sim_Gr_none}}
}
\centerline{
\subfigure[]{\includegraphics[width=0.50\columnwidth]{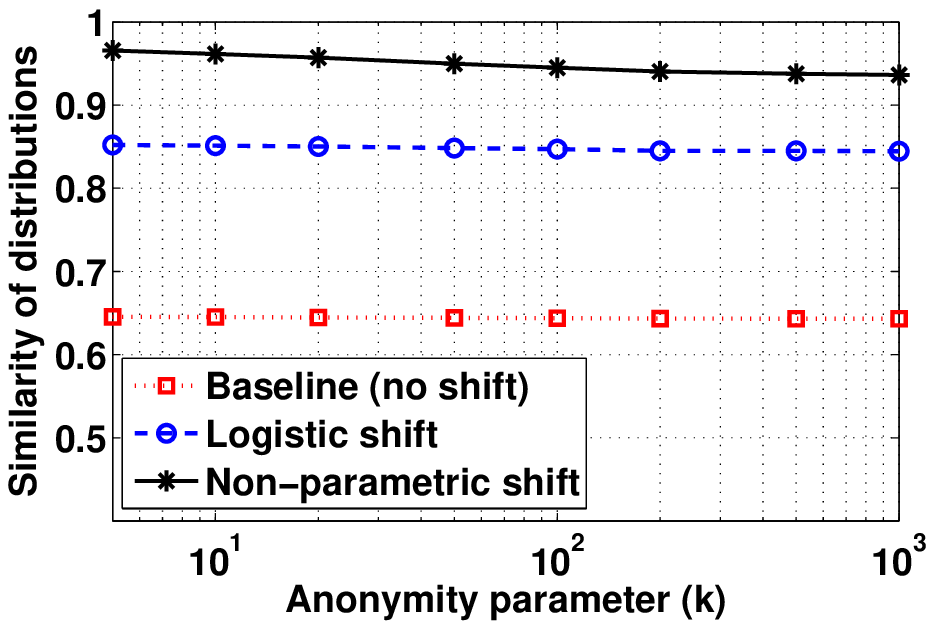}\label{fig:sim_Gr_Re}}
\subfigure[]{\includegraphics[width=0.50\columnwidth]{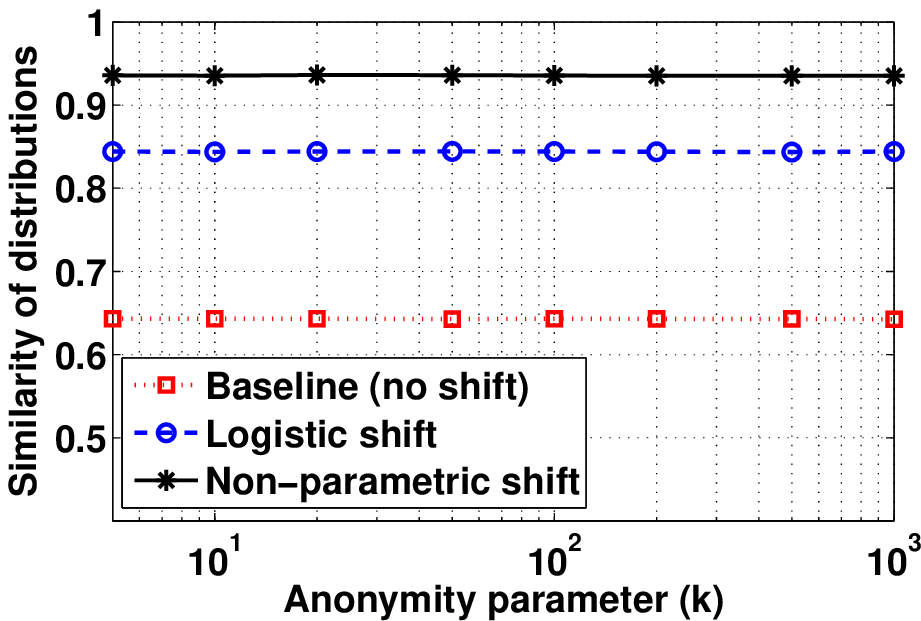}\label{fig:sim_Gr_Ga}}
}
\caption{Area under histogram intersection between the estimated and actual new enrollment distributions for different covariate shift methods and (a) no distribution preservation, (b) intra-cluster dither, or (c) Gaussian dither.  Both proposed distribution-preserving transformations maintain a mostly constant similarity as the anonymity $k$ increases, whereas the similarity deteriorates rapidly without distribution preservation.}
\label{fig:sim}
\vskip -3mm
\end{figure}


Overall, Figures~\ref{fig:Gr_none}--\ref{fig:Gr_Ga} demonstrate the advantage of preserving quasi-identifier distributions when $k$-anonymization is required in a covariate shift setting. 
Some insight into the above results 
can be seen in Figure~\ref{fig:sim}, which depicts the same histogram intersection similarity used in Figure~\ref{fig:noshift}(e), this time between the new enrollment distribution estimated from privacy-transformed training data using the covariate shift methods, and the actual new enrollment.  Under conventional $k$-anonymization in Figure~\ref{fig:sim_Gr_none}, the similarity decreases rapidly with $k$, but using distribution-preserving privacy transformations in Figure~\ref{fig:sim_Gr_Re} and Figure~\ref{fig:sim_Gr_Ga}, the similarity can be kept mostly constant as the anonymity $k$ increases and can be further enhanced by the covariate shift methods.

\section{Conclusion}
\label{sec:conclusion}

In this paper, our main contribution has been to develop a new privacy preservation operation that has the key property of preserving the data distribution of the quasi-identifiers.  The specific privacy criterion we  address is $k$-anonymity, which is a common mathematical interpretation of legal privacy standards.  We have shown how such distribution preservation is a clear need in supervised machine learning settings such as covariate shift and transfer learning that also require anonymized microdata.  Other data analysis workloads that require both anonymization and lack of distortion in the data distribution will also benefit from the proposed methodology.  Workloads requiring distribution preservation have not yet been addressed in the data privacy literature.

Our proposed technique combines $k$-member clustering with components from distribution-preserving quantization, namely dithering and Rosenblatt's transformation.  Distribution-preserving quantization was originally developed in the audio signal processing literature and has never been applied to privacy applications before.  We have analyzed two approaches for dithering, intra-cluster and Gaussian, and shown how both achieve distribution preservation when followed by Rosenblatt's transformation.  Existing distribution-preserving quantization takes the number of clusters as input rather than constraining the cluster size; thus, another contribution of our work is the extension of distribution-preserving quantization to a constrained setting.

Moreover, we have contributed to a solution to the real-world health care market risk assessment problem, a common problem encountered by health insurance companies that was 
especially pertinent after passage of the Affordable Care Act.  Insurance companies had not developed machine learning approaches for this problem, only relying on crude estimates mainly driven by intuition and coarse 
aggregate-level data.  We show successful empirical results on MEPS data with realistic simulations for new markets and enrollment probabilities.  (Actual health cost data from insurance companies that we have worked with in the recent past is confidential.)

In particular, we see that the overall method with intra-cluster dithering keeps the reidentification risk approximately the same as $k$-member clustering, which is sufficient for the requirement of $k$-anonymity.  The Gaussian dither has even lower reidentification risk.  Additionally, the non-parametric version of the covariate shift is successful in significantly reducing the relative bias of the regression that would occur if the covariate shift were not done.  Examining the results of the full solution, we see that without our new privacy preservation methods, cost prediction can 
fail for $k$-anonymity greater than ten or twenty, but prediction results remain satisfactory 
when using our proposed approach.

One health care application-specific direction for future work addresses the following issue.  In the market risk assessment problem, it is possible that the conditional distribution $p_{Y\given X}$ is not the same in the training and test populations, i.e., current and new markets, unlike in the standard covariate shift problem. For example, the overall cost of living in the new market may differ from that in the existing market and this may affect health care costs as well.  However, it is unlikely for there to be sufficient data to learn the full conditional distribution $p_{Y\given X,M}$ (otherwise market shift would not be much of a problem).  One approximation is to assume a simple scaling where an underlying conditional distribution $p_{Y\given X}$ is scaled by a cost-of-living factor $a(M)$ depending on $M$ (implying that the conditional mean for example is $\E[Y\given X,M] = a(M) \E[Y\given X]$). 

A general direction for future work is theoretical analysis of the proposed privacy-preservation method.  Dithered quantization has much supporting theory that we would like to further explore in the context of privacy, where it has never been applied before.  We would also like to explore stronger notions of privacy such as $l$-diversity \cite{machanavajjhala2007diversity} and $t$-closeness \cite{li2007t} as well as non-discrimination and equitability (an issue that is mathematically similar to privacy \cite{Ruggieri2014}), all of which have not been examined for distribution-preserving workloads.  Exploration of different solution methods for the $k$-member clustering optimization problem is of interest as well.


%

\appendix

\section{Proof of Theorem \ref{thm:resampling}}
\label{app:thm}

It suffices to prove that if $\tilde{x} \in \mathcal{C}(i_1^\ast,\dots,i_d^\ast)$, then Rosenblatt's transformation yields $\hat{x} = v(i_1^\ast,\dots,i_d^\ast)$ with probability $1$ for any $(i_1^\ast,\dots,i_d^\ast)$ so that the cell probabilities in \eqref{eq:probCell} translate directly into the desired resampling probabilities.  This can be done by induction over the dimensions $j = 1,\dots,d$. 

For $j = 1$, we consider the CDF $F_{\tilde{X}_1}(\tilde{x}_1)$ used in the first step of transformation \eqref{eq:RosenblattForward}.  For any $i_1^\ast = 1, \dots, n_1$, the law of total probability gives 
\begin{displaymath}
\Pr\bigl(\tilde{X}_1 \in \mathcal{I}_1(i_1^\ast)\bigr) = \sum_{\ell=1}^c p_L(\ell) \sum_{i_1,\dots,i_d} \Pr\bigl(\tilde{X} \in \mathcal{C}(i_1,\dots,i_d) \given \ell\bigr)\Pr\bigl(\tilde{X}_1 \in \mathcal{I}_1(i_1^\ast) \given \tilde{X} \in \mathcal{C}(i_1,\dots,i_d) \bigr).
\end{displaymath}
Substituting \eqref{eq:probCell} and using the fact that 
\[
\Pr\bigl(\tilde{X}_1 \in \mathcal{I}_1(i_1^\ast) \given \tilde{X} \in \mathcal{C}(i_1,\dots,i_d) \bigr) = 
\begin{cases}
1, & i_1 = i_1^\ast,\\
0, & i_1 \neq i_1^\ast,
\end{cases}
\]
we obtain
\begin{align}
\Pr\bigl(\tilde{X}_1 \in \mathcal{I}_1(i_1^\ast)\bigr) &= \sum_{\ell=1}^c \frac{n_\ell}{n} \sum_{i_2,\dots,i_d} \frac{n_{\ell}(i_1^\ast,i_2,\dots,i_d)}{n_\ell}\nonumber\\
&= \sum_{i_2,\dots,i_d} \frac{n(i_1^\ast,i_2,\dots,i_d)}{n}\nonumber\\
&= p_{X_1}\bigl(v_1(i_1^\ast)\bigr)\label{eq:thm:resampling1}
\end{align}
by the definition of empirical distribution.  It follows that for $\tilde{x}_1 \in \mathcal{I}_1(i_1^\ast)$,
\begin{align}
\sum_{i_1=1}^{i_1^\ast-1} p_{X_1}\bigl(v_1(i_1)\bigr) \leq &F_{\tilde{X}_1}(\tilde{x}_1) \leq \sum_{i_1=1}^{i_1^\ast} p_{X_1}\bigl(v_1(i_1)\bigr),\nonumber\\
F_{X_1}\bigl(v_1(i_1^\ast - 1)\bigr) \leq &F_{\tilde{X}_1}(\tilde{x}_1) \leq F_{X_1}\bigl(v_1(i_1^\ast)\bigr).\label{eq:thm:resampling2}
\end{align}
Combining \eqref{eq:thm:resampling2} with \eqref{eq:RosenblattForward} and \eqref{eq:RosenblattInverse2} for $j=1$, we conclude that $\hat{x}_1 = v_1(i_1^\ast)$ with probability $1$.

For $j > 1$, we consider the conditional distribution of $\tilde{X}_j$ given $\tilde{X}^{j-1} = \tilde{x}^{j-1}$.  Let $f$ denote a generic probability density function (PDF).  Starting from 
\[
\Pr\bigl(\tilde{X}_j \in \mathcal{I}_j(i_j^\ast) \given \tilde{X}^{j-1} = \tilde{x}^{j-1}\bigr) = \frac{\displaystyle\int_{\mathcal{I}_j(i_j^\ast)} f_{\tilde{X}^{j}}(\tilde{x}^{j}) d\tilde{x}_j}{f_{\tilde{X}^{j-1}}(\tilde{x}^{j-1})},
\]
we apply a similar total probability decomposition as above to both numerator and denominator to obtain 
\begin{align}
&\Pr\bigl(\tilde{X}_j \in \mathcal{I}_j(i_j^\ast) \given \tilde{X}^{j-1} = \tilde{x}^{j-1}\bigr)\nonumber\\ 
&= \frac{\displaystyle\sum_{i_1,\dots,i_d} \frac{n(i_1,\dots,i_d)}{n} \int_{\mathcal{I}_j(i_j^\ast)} f_{\tilde{X}^{j}\given C}\bigl(\tilde{x}^{j} \given C(i_1,\dots,i_d) \bigr) d\tilde{x}_j}{\displaystyle\sum_{i_1,\dots,i_d} \frac{n(i_1,\dots,i_d)}{n} f_{\tilde{X}^{j-1}\given C}\bigl(\tilde{x}^{j-1} \given C(i_1,\dots,i_d) \bigr)}.\label{eq:thm:resampling3}
\end{align}
Since $\tilde{x}_{j'} \in \mathcal{I}_{j'}(i_{j'}^\ast)$ for $j' < j$, the PDFs in the numerator and denominator of \eqref{eq:thm:resampling3} are zero unless $i_{j'} = i_{j'}^{\ast}$, $j' < j$, while in the numerator, the integral is also zero unless $i_j = i_j^\ast$, in which case $\tilde{x}_j$ is marginalized out.  Hence the numerator becomes 
\begin{displaymath}
\sum_{i_{j+1},\dots,i_d} \frac{n(i_1^\ast,\dots,i_j^\ast,i_{j+1},\dots,i_d)}{n} f_{\tilde{X}^{j-1}\given C}\bigl(\tilde{x}^{j-1} \given C(i_1^\ast,\dots,i_j^\ast,i_{j+1},\dots,i_d) \bigr),
\end{displaymath}
while the denominator is similar except for an additional sum over $i_j$.  We now exploit the fact that $\tilde{X}^{j-1}$ is uniformly distributed conditioned on a cell, i.e.\ with $\abs{\mathcal{I}}$ denoting the width of interval $\mathcal{I}$, 
\begin{displaymath}
f_{\tilde{X}^{j-1}\given C}\bigl(\tilde{x}^{j-1} \given C(i_1^\ast,\dots,i_{j-1}^\ast,i_{j},\dots,i_d) \bigr) = \left(\prod_{j'=1}^{j-1} \abs{\mathcal{I}_{j'}(i_{j'}^\ast)} \right)^{-1}, \quad \forall \; i_j, \dots, i_d.
\end{displaymath}
This implies that all remaining PDFs in the numerator and denominator have the same value, reducing \eqref{eq:thm:resampling3} to 
\begin{align}
\Pr\bigl(\tilde{X}_j &\in \mathcal{I}_j(i_j^\ast) \given \tilde{X}^{j-1} = \tilde{x}^{j-1}\bigr)\nonumber\\
&= \frac{\displaystyle\sum_{i_{j+1},\dots,i_d} \frac{n(i_1^\ast,\dots,i_j^\ast,i_{j+1},\dots,i_d)}{n}}{\displaystyle\sum_{i_{j},\dots,i_d} \frac{n(i_1^\ast,\dots,i_{j-1}^\ast,i_{j},\dots,i_d)}{n}}\nonumber\\
&= \frac{p_{X^j}\bigl( v^j(i_1^\ast,\dots,i_j^\ast) \bigr)}{p_{X^{j-1}}\bigl( v^{j-1}(i_1^\ast,\dots,i_{j-1}^\ast) \bigr)}\nonumber\\
&= p_{X_j\given X^{j-1}}\bigl( v_j(i_j^\ast) \given v^{j-1}(i_1^\ast,\dots,i_{j-1}^\ast) \bigr).\label{eq:thm:resampling4}
\end{align}
Similar to \eqref{eq:thm:resampling1} and \eqref{eq:thm:resampling2}, \eqref{eq:thm:resampling4} implies that for $\tilde{x}_{j'} \in \mathcal{I}_{j'}(i_{j'}^\ast)$, $j' = 1,\dots,j$,
\begin{align*}
F_{X_j\given X^{j-1}}&\bigl( v_j(i_j^\ast-1) \given v^{j-1}(i_1^\ast,\dots,i_{j-1}^\ast) \bigr)\\
&\leq F_{\tilde{X}_j\given\tilde{X}^{j-1}}\bigl(\tilde{x}_j \given \tilde{x}^{j-1}\bigr)\\
&\leq F_{X_j\given X^{j-1}}\bigl( v_j(i_j^\ast) \given v^{j-1}(i_1^\ast,\dots,i_{j-1}^\ast) \bigr).
\end{align*}
We combine these last inequalities with \eqref{eq:RosenblattForward} and \eqref{eq:RosenblattInverse2}, where $\hat{x}^{j-1} = v^{j-1}(i_1^\ast,\dots,i_{j-1}^\ast)$ by the inductive assumption.  This establishes that $\hat{x}_j = v_j(i_j^\ast)$ with probability $1$.


  \section*{Acknowledgment}

The authors thank A.~Gkoulalas-Divanis, Y.~Hwang, V.~S.~Iyengar,  A.~Mojsilovi\'c, and G.~Yuen-Reed for conversations and support.




\bibliographystyle{IEEEtran}
\bibliography{jkanonymity}

\begin{thebibliography}{10}
\providecommand{\url}[1]{#1}
\csname url@samestyle\endcsname
\providecommand{\newblock}{\relax}
\providecommand{\bibinfo}[2]{#2}
\providecommand{\BIBentrySTDinterwordspacing}{\spaceskip=0pt\relax}
\providecommand{\BIBentryALTinterwordstretchfactor}{4}
\providecommand{\BIBentryALTinterwordspacing}{\spaceskip=\fontdimen2\font plus
\BIBentryALTinterwordstretchfactor\fontdimen3\font minus
  \fontdimen4\font\relax}
\providecommand{\BIBforeignlanguage}[2]{{%
\expandafter\ifx\csname l@#1\endcsname\relax
\typeout{** WARNING: IEEEtran.bst: No hyphenation pattern has been}%
\typeout{** loaded for the language `#1'. Using the pattern for}%
\typeout{** the default language instead.}%
\else
\language=\csname l@#1\endcsname
\fi
#2}}
\providecommand{\BIBdecl}{\relax}
\BIBdecl

\bibitem{Sweeney2002}
L.~Sweeney, ``k-anonymity: A model for protecting privacy,'' \emph{Int. J.
  Uncertain. Fuzz.}, vol.~10, no.~5, pp. 557--570, Oct. 2002.

\bibitem{Samarati2001}
P.~Samarati, ``Protecting respondents' identities in microdata release,''
  \emph{{IEEE} Trans. Knowl. Data Eng.}, vol.~13, no.~6, pp. 1010--1027,
  Nov./Dec. 2001.

\bibitem{Iyengar2002}
V.~S. Iyengar, ``Transforming data to satisfy privacy constraints,'' in
  \emph{Proc. ACM SIGKDD Int. Conf. Knowl. Disc. Data Min.}, Edmonton, Canada,
  Jul. 2002, pp. 279--288.

\bibitem{YiWJ2014}
J.~Yi, J.~Wang, and R.~Jin, ``Privacy and regression model preserved
  learning,'' in \emph{Proc. AAAI Conf. Artif. Intell.}, Qu{\'e}bec City,
  Canada, Jul. 2014, pp. 1341--1347.

\bibitem{SugiyamaKM2007}
M.~Sugiyama, M.~Krauledat, and K.-R. M{\"u}ller, ``Covariate shift adaptation
  by importance weighted cross validation,'' \emph{J. Mach. Learn. Res.},
  vol.~8, pp. 985--1005, May 2007.

\bibitem{BickelSS2009}
S.~Bickel, C.~Sawade, and T.~Scheffer, ``Transfer learning by distribution
  matching for targeted advertising,'' in \emph{Adv. Neur. Inf. Process. Syst.
  21}, 2009, pp. 145--152.

\bibitem{Quionero-CandelaSSL2009}
J.~Qui{\~{n}}onero-Candela, M.~Sugiyama, A.~Schwaighofer, and N.~D. Lawrence,
  Eds., \emph{Dataset Shift in Machine Learning}.\hskip 1em plus 0.5em minus
  0.4em\relax Cambridge, MA: MIT Press, 2009.

\bibitem{WeiRV2015}
D.~Wei, K.~N. Ramamurthy, and K.~R. Varshney, ``Health insurance market risk
  assessment: Covariate shift and k-anonymity,'' in \emph{Proc. SIAM Int. Conf.
  Data Min.}, Vancouver, Canada, Apr.--May 2015.

\bibitem{BalkanBFM2012}
M.-F. Balcan, A.~Blum, S.~Fine, and Y.~Mansour, ``Distributed learning,
  communication complexity and privacy,'' in \emph{Proc. Conf. Learn. Theory},
  Edinburgh, UK, Jun. 2012, p.~26.

\bibitem{DariesRWYWHSC2014}
J.~P. Daries, J.~Reich, J.~Waldo, E.~M. Young, J.~Whittinghill, A.~D. Ho, D.~T.
  Seaton, and I.~Chuang, ``Privacy, anonymity, and big data in the social
  sciences,'' \emph{Commun. ACM}, vol.~57, no.~9, pp. 56--63, Sep. 2014.

\bibitem{MalinBM2011}
B.~Malin, K.~Benitez, and D.~Masys, ``Never too old for anonymity: A
  statistical standard for demographic data sharing via the {HIPAA} privacy
  rule,'' \emph{J. Am. Med. Inform. Assoc.}, vol.~18, no.~1, pp. 3--10, Jan.
  2011.

\bibitem{ElemamD2008}
K.~El~Emam and F.~K. Dankar, ``Protecting privacy using k-anonymity,'' \emph{J.
  Am. Med. Inform. Assoc.}, vol. 157, no.~5, pp. 627--637, Sep./Oct. 2008.

\bibitem{LeFevreDR2006}
K.~LeFevre, D.~J. DeWitt, and R.~Ramakrishnan, ``Mondrian multidimensional
  k-anonymity,'' in \emph{Proc. Int. Conf. Data Eng.}, Atlanta, GA, Apr. 2006.

\bibitem{ByunKBL2007}
J.-W. Byun, A.~Kamra, E.~Bertino, and N.~Li, ``Efficient k-anonymization using
  clustering techniques,'' in \emph{Proc. Int. Conf. Database Syst. Adv.
  Appl.}, Bangkok, Thailand, Apr. 2007, pp. 188--200.

\bibitem{GrayN1998}
R.~M. Gray and D.~L. Neuhoff, ``Quantization,'' \emph{{IEEE} Trans. Inf.
  Theory}, vol.~44, no.~6, pp. 2325--2383, Oct. 1998.

\bibitem{LiKK2010}
M.~Li, J.~Klejsa, and W.~B. Kleijn, ``Distribution preserving quantization with
  dithering and transformation,'' \emph{{IEEE} Signal Process. Lett.}, vol.~17,
  no.~12, pp. 1014--1017, Dec. 2010.

\bibitem{AlamgirLL2014}
M.~Alamgir, G.~Lugosi, and U.~von Luxburg, ``Density-preserving quantization
  with application to graph downsampling,'' in \emph{Proc. Conf. Learn.
  Theory}, Barcelona, Spain, Jun. 2014, pp. 543--559.

\bibitem{LipshitzWV1992}
S.~P. Lipshitz, R.~A. Wannamaker, and J.~Vanderkooy, ``Quantization and dither:
  A theoretical survey,'' \emph{J. Audio Eng. Soc.}, vol.~40, no.~5, pp.
  355--375, May 1992.

\bibitem{Rosenblatt1952}
M.~Rosenblatt, ``Remarks on a multivariate transformation,'' \emph{Ann. Math.
  Stat.}, vol.~23, no.~3, pp. 470--472, Sep. 1952.

\bibitem{KarguptaDWS2003}
H.~Kargupta, S.~Datta, Q.~Wang, and K.~Sivakumar, ``On the privacy preserving
  properties of random data perturpation techniques,'' in \emph{Proc. IEEE Int.
  Conf. Data Min.}, Melbourne, FL, Nov. 2003, pp. 99--106.

\bibitem{DemirizBB2009}
A.~Demiriz, K.~P. Bennett, and P.~S. Bradley, ``Using assignment constraints to
  avoid empty clusters in k-means clustering,'' in \emph{Constrained
  Clustering: Advances in Algorithms, Theory, and Applications}, S.~Basu,
  I.~Davidson, and K.~Wagstaff, Eds.\hskip 1em plus 0.5em minus 0.4em\relax
  Boca Raton, FL: CRC Press, 2009, pp. 201--220.

\bibitem{Rebollo-MonederoFPP2013}
D.~Rebollo-Monedero, J.~Forn{\'e}, E.~Pallar{\`e}s, and J.~Parra-Arnau, ``A
  modification of the {L}loyd algorithm for $k$-anonymous quantization,''
  \emph{Inf. Sci.}, vol. 222, pp. 185--202, 10 Feb. 2013.

\bibitem{BanerjeeG2006}
A.~Banerjee and J.~Ghosh, ``Scalable clustering algorithms with balancing
  constraints,'' \emph{Data Min. Knowl. Disc.}, vol.~13, no.~3, pp. 365--395,
  Nov. 2006.

\bibitem{GeEJD2007}
R.~Ge, M.~Ester, W.~Jin, and I.~Davidson, ``Constraint-driven clustering,'' in
  \emph{Proc. ACM SIGKDD Int. Conf. Knowl. Disc. Data Min.}, San Jose, CA, Aug.
  2007, pp. 320--329.

\bibitem{GeethaPV2009}
S.~Geetha, G.~Poonthalir, and P.~T. Vanathi, ``Improved k-means algorithm for
  capacitated clustering problem,'' \emph{INFOCOMP J. Comp. Sci.}, vol.~8,
  no.~4, pp. 52--59, Dec. 2009.

\bibitem{GanganathCT2014}
N.~Ganganath, C.-T. Cheng, and C.~K. Tse, ``Data clustering with cluster size
  constraints using a modified $k$-means algorithm,'' in \emph{Proc. Int. Conf.
  Cyber-Enabled Distr. Comput. Knowl. Disc.}, Shanghai, China, Oct. 2014, pp.
  158--161.

\bibitem{BasuDW2009}
S.~Basu, I.~Davidson, and K.~Wagstaff, \emph{Constrained Clustering: Advances
  in Algorithms, Theory, and Applications}.\hskip 1em plus 0.5em minus
  0.4em\relax Boca Raton, FL: CRC Press, 2009.

\bibitem{Messerschmitt1971}
D.~G. Messerschmitt, ``Quantizing for maximum output entropy,'' \emph{{IEEE}
  Trans. Inf. Theory}, vol. IT-17, no.~5, p. 612, Sep. 1971.

\bibitem{DiehrYAHL1999}
P.~Diehr, D.~Yanez, A.~Ash, M.~Hornbrook, and D.~Y. Lin, ``Methods for
  analyzing health care utilization and costs,'' \emph{Annu. Rev. Public
  Health}, vol.~20, pp. 125--144, 1999.

\bibitem{BasuM2009}
A.~Basu and W.~G. Manning, ``Issues for the next generation of health care cost
  analyses,'' \emph{Med. Care}, vol.~47, no.~7, pp. S109--S114, Jul. 2009.

\bibitem{WeiRKM2014}
D.~Wei, K.~N. Ramamurthy, D.~A. Katz-Rogozhnikov, and A.~Mojsilovi{\'c},
  ``Multiplicative regression via constrained least squares,'' in \emph{Proc.
  IEEE Workshop Stat. Signal Process.}, Gold Coast, Australia, Jun.--Jul. 2014,
  pp. 304--307.

\bibitem{ACS2005}
``American community survey,''
  http://www.census.gov/acs/www/data\_documentation/pums\_data, United States
  Census Bureau, 2005.

\bibitem{HHS2014}
``Enrollment in the health insurance marketplace totals over 8 million
  people,'' United States Department of Health and Human Services, Press
  Release, May 2014.

\bibitem{ASPE2014}
``Health insurance marketplace: Summary enrollment report for the initial
  annual open enrollment,'' United States Department of Health and Human
  Services, {ASPE} Issue Brief, May 2014.

\bibitem{Shao2003}
J.~Shao, \emph{Mathematical Statistics}, 2nd~ed.\hskip 1em plus 0.5em minus
  0.4em\relax New York, NY: Springer, 2003.

\bibitem{swain1991color}
M.~J. Swain and D.~H. Ballard, ``Color indexing,'' \emph{Int. J. Comput. Vis.},
  vol.~7, no.~1, pp. 11--32, 1991.

\bibitem{machanavajjhala2007diversity}
A.~Machanavajjhala, D.~Kifer, J.~Gehrke, and M.~Venkitasubramaniam,
  ``l-diversity: Privacy beyond k-anonymity,'' \emph{ACM Trans. Knowl. Disc.
  Data}, vol.~1, no.~1, p.~3, 2007.

\bibitem{li2007t}
N.~Li, T.~Li, and S.~Venkatasubramanian, ``t-closeness: Privacy beyond
  k-anonymity and l-diversity.'' in \emph{Proc. IEEE Int. Conf. Data Eng.},
  Istanbul, Turkey, Apr. 2007, pp. 106--115.

\bibitem{Ruggieri2014}
S.~Ruggieri, ``Using t-closeness anonymity to control for non-discrimination,''
  \emph{Trans. Data Privacy}, vol.~7, no.~2, pp. 99--129, Aug. 2014.

\end{thebibliography}
\end{document}